\documentclass[12pt]{article}
\usepackage[margin=1.25in]{geometry}

\usepackage{tikz}
\usepackage{tikz-qtree}
\usepackage[page, header]{appendix}
\usepackage{titletoc}
\usepackage{amsmath, amsfonts, amssymb, amsthm, amsbsy, amscd, bm, bbm,mathrsfs} 
\usepackage{color}
\usepackage[colorlinks,linkcolor=blue,citecolor=blue]{hyperref}
\usepackage{cleveref}
\usepackage{times}
\usepackage{algpseudocode,algorithm}
\usepackage{thmtools}
\usepackage{thm-restate}
\usepackage{verbatim}

\newtheorem{assumption}{Assumption}

\def\iid{\textrm{i.i.d.~}} 
\def\wp{\textit{w.p.~}}

\newcommand{\rad}{\operatorname{Rad}_n}
\newcommand{\erad}{\widehat{\operatorname{Rad}}}

\newcommand{\EE}{\operatorname{\mathbb{E}}}

\newcommand{\NN}{\mathbb{N}}

\newcommand{\PP}{\mathbb{P}}

\newcommand{\RR}{\mathbb{R}}
\renewcommand{\SS}{\mathbb{S}}

\DeclareMathOperator*{\argmin}{argmin}

\renewcommand{\leq}{\leqslant}
\renewcommand{\geq}{\geqslant}

\newcommand{\cD}{\mathcal{D}}

\newcommand{\cF}{\mathcal{F}}
\newcommand{\cG}{\mathcal{G}}
\newcommand{\cH}{\mathcal{H}}

\newcommand{\cN}{\mathcal{N}}
\newcommand{\cO}{\mathcal{O}}
\newcommand{\cP}{\mathcal{P}}

\newcommand{\half}{\frac{1}{2}}

\newcommand{\tr}{{\rm tr}}

\newcommand{\abs}[1]{\left\vert#1\right\vert}
\newcommand{\norm}[2]{\left\Vert#1\right\Vert_{#2}}

\newcommand{\ep}{{\epsilon}}

\newcommand{\zh}[1]{\textsf{\textcolor{teal}{[zh: #1]}}}

\newtheorem{theorem}{Theorem}
\newtheorem{corollary}{Corollary}

\newtheorem{lemma}{Lemma}

\theoremstyle{definition}
\newtheorem{definition}{Definition}

\newtheorem{remark}{Remark}

\title{Learning Hierarchical Polynomials with\\ Three-Layer Neural Networks}
\author{
  Zihao Wang\\
  Peking University\\
  \texttt{zihaowang@stu.pku.edu.cn}
  \and
  Eshaan Nichani\\
  Princeton University\\
  \texttt{eshnich@princeton.edu}
  \and
  Jason D. Lee\\
  Princeton University\\
  \texttt{jasonlee@princeton.edu}
}

\date{\today\vspace{-1em}}

\ifdefined\usebigfont

\usepackage{times}
\usepackage[fontsize=13pt]{scrextend}
\makeatletter
\@ifpackageloaded{geometry}{\AtBeginDocument{\newgeometry{letterpaper,left=1.56in,right=1.56in,top=1.71in,bottom=1.77in}}}{\usepackage[letterpaper,left=1.56in,right=1.56in,top=1.71in,bottom=1.77in]{geometry}}
\AtBeginDocument{\newgeometry{letterpaper,left=1.56in,right=1.56in,top=1.71in,bottom=1.77in}}
\makeatother
\else
\usepackage{times}
\fi

\usepackage[square]{natbib}
\begin{document}

\maketitle 

\begin{abstract}
We study the problem of learning hierarchical polynomials over the standard Gaussian distribution with three-layer neural networks. We specifically consider target functions of the form $h = g \circ p$ where $p : \mathbb{R}^d \rightarrow \mathbb{R}$ is a degree $k$ polynomial and $g: \mathbb{R} \rightarrow \mathbb{R}$ is a degree $q$ polynomial. This function class generalizes the single-index model, which corresponds to $k=1$, and is a natural class of functions possessing an underlying hierarchical structure. Our main result shows that for a large subclass of degree $k$ polynomials $p$, a three-layer neural network trained via layerwise gradient descent on the square loss learns the target $h$ up to vanishing test error in $\widetilde \cO(d^k)$ samples and polynomial time. This is a strict improvement over kernel methods, which require $\widetilde \Theta(d^{kq})$ samples, as well as existing guarantees for two-layer networks, which require the target function to be low-rank. Our result also generalizes prior works on three-layer neural networks, which were restricted to the case of $p$ being a quadratic. When $p$ is indeed a quadratic, we achieve the information-theoretically optimal sample complexity $\widetilde \cO(d^2)$, which is an improvement over prior work~\citep{nichani2023provable} requiring a sample size of $\widetilde\Theta(d^4)$. Our proof proceeds by showing that during the initial stage of training the network performs feature learning to recover the feature $p$ with $\widetilde \cO(d^k)$ samples. This work demonstrates the ability of three-layer neural networks to learn complex features and as a result, learn a broad class of hierarchical functions.

\end{abstract}

\section{Introduction}

Deep neural networks have demonstrated impressive empirical successes across a wide range of domains. This improved accuracy and the effectiveness of the modern pretraining and finetuning paradigm is often attributed to the ability of neural networks to efficiently learn input features from data. On ``real-world" learning problems posited to be hierarchical in nature, conventional wisdom is that neural networks first learn salient input features to more efficiently learn hierarchical functions depending on these features. This feature learning capability is hypothesized to be a key advantage of neural networks over fixed-feature approaches such as kernel methods~\citep{wei2020regularization,allenzhu2020resnet,bai2020linearization}.

Recent theoretical work has sought to formalize this notion of a hierarchical function and understand the process by which neural networks learn features. These works specifically study which classes of hierarchical functions can be efficiently learned via gradient descent on a neural network, with a sample complexity improvement over kernel methods or shallower networks that cannot utilize the hierarchical structure. The most common such example is the \emph{multi-index model}, in which the target $f^*$ depends solely on the projection of the data onto a low-rank subspace, i.e $f^*(x) = g(Ux)$ for a projection matrix $U \in \RR^{r \times d}$ and unknown link function $g: \RR^r \rightarrow \RR$. Here, a hierarchical learning process simply needs to extract the hidden subspace $U$ and learn the $r$-dimensional function $g$. Prior work~\citep{abbe2022mergedstaircase, abbe2023sgd, damian2022neural, bietti2022learning} shows that two-layer neural networks trained via gradient descent indeed learn the low-dimensional feature $Ux$, and thus learn multi-index models with an improved sample complexity over kernel methods. 

Beyond the multi-index model, there is growing work on the ability of deeper neural networks to learn more general classes of hierarchical functions. \citep{safran2022optimization, ren2023depth, nichani2023provable} show that three-layer networks trained with variants of gradient descent can learn hierarchical targets of the form $h = g\circ p$, where $p$ is a simple nonlinear feature such as the norm $p(x) = \|x\|_2$ or a quadratic $p(x) = x^{\top}Ax$. However, it remains an open question to understand whether neural networks can more efficiently learn a broader class of hierarchical functions.

\subsection{Our Results}

In this work, we study the problem of learning \emph{hierarchical polynomials} over the standard $d$-dimensional Gaussian distribution. Specifically, we consider learning the target function $h: \mathbb{R}^d \rightarrow \mathbb{R}$, where $h$ is equipped with the hierarchical structure $h = g \circ p$ for polynomials $g : \RR \rightarrow \RR$ and $p: \RR^d \rightarrow \RR$ of degree $q$ and $k$ respectively. This class of functions is a generalization of the single-index model, which corresponds to $k=1$.

Our main result, \Cref{thm:main_thm}, is that for a large class of degree $k$ polynomials $p$, a three-layer neural network trained via layer-wise gradient descent can efficiently learn the hierarchical polynomial $h = g \circ p$ in $\widetilde \cO(d^k)$ samples. Crucially, this sample complexity is a significant improvement over learning $h$ via a kernel method, which requires $\widetilde \Omega(d^{qk})$ samples~\citep{ghorbani2021linearized}. Our high level insight is that the sample complexity of learning $g\circ p$ is the same as that of learning the feature $p$, as $p$ can be extracted from the low degree terms of $g\circ p$. Since neural networks learn in increasing complexity~\citep{abbe2022mergedstaircase,abbe2023sgd,xu2020frequency}, such learning process is easily implemented by GD on a three-layer neural network. We verify this insight both theoretically via our layerwise training procedure (\Cref{alg:layerwise}) and empirically via simulations in Section \ref{app:experiments}.


Our proof proceeds by showing that during the initial stage of training the network implements kernel regression in $d$-dimensions to learn the feature $p$ \emph{even though it only sees $g \circ p$,} and in the next stage implements 1D kernel regression to fit the link function $g$. This feature learning during the initial stage relies on showing that the low-frequency component of the target function $g\circ p$ is approximately proportional to the feature $p$, by the ``approximate Stein's Lemma" stated in \Cref{lem:approximate_stein}, which is our main technical contribution. This demonstrates that three-layer networks trained with gradient descent, unlike kernel methods, do allow for adaptivity and thus the ability to learn features.


\subsection{Related Works}

\paragraph{Kernel Methods.} Initial learning guarantees for neural networks relied on the Neural Tangent Kernel (NTK) approach, which couples GD dynamics to those of the network's linearization about the initialization~\citep{jacot2018neural, soltanolkotabi2018theoretical, du2018gradient, chizat2019lazy}. However, the NTK theory fails to capture the success of neural networks in practice~\citep{arora2019exact, lee2020finite,e2020a}. Furthermore, \citet{ghorbani2021linearized} presents a lower bound showing that for data uniform on the sphere, the NTK requires $\widetilde\Omega(d^k)$ samples to learn any degree $k$ polynomial in $d$ dimensions. Crucially, networks in the kernel regime cannot learn features~\citep{yang2021tensor}, and hence cannot adapt to low-dimensional structure. An important question is thus to understand how neural networks are able to adapt to underlying structures in the target function and learn salient features, which allow for improved generalization over kernel methods.

\paragraph{Two-layer Neural Networks.} Recent work has studied the ability of two-layer neural networks to learn features and as a consequence learn hierarchical functions with a sample complexity improvement over kernel methods. For isotropic data, two-layer neural networks are capable of efficiently learning multi-index models, i.e. functions of the form $f^*(x) = g(Ux)$. Specifically, for Gaussian covariates, \citet{damian2022neural, abbe2023sgd, dandi2023learning} show that two-layer neural networks learn low-rank polynomials with a sample complexity whose dimension dependence does not scale with the degree of the polynomial, and \citet{bietti2022learning, ba2022high} show two-layer networks efficiently learn single-index models. For data uniform on the hypercube, \citet{abbe2022mergedstaircase} shows learnability of a special class of sparse boolean functions in $\cO(d)$ steps of SGD. These prior works rely on layerwise training procedures which learn the relevant subspace in the first stage, and fit the link function $g$ in the second stage. Relatedly, fully connected networks trained via gradient descent on standard image classification tasks have been shown to learn such relevant low-rank features~\citep{lee2007sparse,radhakrishnan2022feature}.


\paragraph{Three-layer Neural Networks.} Prior work has also shown that three-layer neural networks can learn certain classes of hierarchical functions. \citet{chen2020towards} shows that three-layer networks can more efficiently learn low-rank polynomials by decomposing the function $z^p$ as $(z^{p/2})^2$. \citet{allen2019learning} uses a modified version of GD to improperly learn a class of three-layer networks via a second-order variant of the NTK. \citet{safran2022optimization} shows that certain ball indicator functions of the form $\mathbf{1}_{\|x\| \geq \lambda}$ are efficiently learnable via GD on a three-layer network. They accompany this with a lower bound showing that such targets are not even approximatable by polynomially-sized two-layer networks. \citet{ren2023depth} shows that a multi-layer mean-field network can learn the target $\mathrm{ReLU}(1 - \|x\|)$. Our work considers a broader class of hierarchical functions and features.

Our work is most similar to \citet{allen2019can, allen2020backward, nichani2023provable}. \citet{allen2019can} considers learning target functions of the form $p + \alpha g \circ p$ with a three-layer residual network similar our architecture \eqref{eq:three-layer-nn}. They consider a similar hierarchical learning procedure where the first layer learns $p$ while the second learns $g$. However \citet{allen2019can} can only learn the target up to $O(\alpha^4)$ error, while our analysis shows learnability of targets of the form $g \circ p$, corresponding to $\alpha = \Theta(1)$, up to $o_d(1)$ error. \citet{allen2020backward} shows that a deeper network with quadratic activations learns a similar class of hierarchical functions up to arbitrarily small error, but crucially requires $\alpha$ to be $o_d(1)$. We remark that our results do require Gaussianity of the input distribution, while \citet{allen2019can, allen2020backward} hold for a more general class of data distributions. \citet{nichani2023provable} shows that a three-layer network trained with layerwise GD, where the first stage consists of a single gradient step, efficiently learns the hierarchical function $g \circ p$ when $p$ is a quadratic, with width and sample complexity $\widetilde\Theta(d^4)$. Our \Cref{thm:main_thm} extends this result to the case where $p$ is a degree $k$ polynomial. Furthermore, when $p$ is quadratic, \Cref{cor:quad} shows that our algorithm only requires a width and sample complexity of $\widetilde\Theta(d^2)$, which matches the information-theoretic lower bound. Our sample complexity improvement for quadratic features relies on showing that running gradient descent for multiple steps can more efficiently extract the feature $p$ during the feature learning stage. Furthermore, the extension to degree $k$ polynomial features relies on a generalization of the approximate Stein's lemma, a key technical innovation of our work. 




\subsection{Notations}

We let $\sum_{i_j}$ denote the sum over increasing sequences $(i_1, \dots i_z)$, i.e $\sum_{i_1<i_2<\dots<i_z}$. 
We use $X \lesssim Y$ to denote $X\leq CY$ for some absolute positive constant $C$ and $X\gtrsim Y$ is defined analogously.
We use $\operatorname{poly}(z_1,\dots,z_p)$ to denote a quantity that depends on $z_1,\dots,z_p$ polynomially.
We also use the standard big-O notations: $\Theta(\cdot)$, $\cO(\cdot)$ and $\Omega(\cdot)$ to only hide absolute positive constants. In addition, we use $\widetilde{\cO}$ and $\widetilde{\Omega}$ to hide higher-order terms, e.g., $\cO((\log d)(\log \log d)^2)=\widetilde{\cO}(\log d)$ and $\cO(d\log d) = \widetilde{\cO}(d)$. Let  $a\wedge b=\min(a,b)$, $[k]=\{1,2,\dots,k\}$ for $k\in \NN$.
For a vector $v$, denote by $\|v\|_{p}:=(\sum_i |v_i|^p)^{1/p}$ the $\ell^p$ norm. When $p=2$, we omit the subscript for simplicity.
For a matrix $A$, let $\|A\|$ and $\|A\|_F$ be the spectral norm and Frobenius norm, respectively.
We use $\lambda_{\operatorname{max}}(\cdot)$ and $\lambda_{\operatorname{min}}(\cdot)$ to denote the maximal and the minimal eigenvalue of a real symmetric matrix. For a vector $w \in \RR^R$ and $k\leq R$, we use $w_{\leq k}\in \RR^k$ to denote the first $k$ coordinates of $w$ and $w_{>k}$ to denote the last $R-k$ coordinates of $w$. That is to say, we can write $w=(w_{\leq k},w_{>k})$.
\section{Preliminaries}

\subsection{Problem Setup}

Our aim is to learn the target function $h:\RR^d\rightarrow\RR$, where $\RR^d$ is the input domain equipped with the standard normal distribution $\gamma := \cN(0,I_d)$.
We assume our target has a compositional structure, that is to say, $h=g\circ p$ for some $g:\RR\rightarrow\RR$ and $p:\RR^d\rightarrow\RR$.

\begin{assumption}
    $p$ is a degree $k$ polynomial with $k\geq 2$, and $g$ is a degree $q$ polynomial.
\end{assumption}
The degree of $h$ is at most $r:=kq$. We treat $k,q$ as absolute constants, and hide constants that depend only on $k,q$ using big-O notation. We require the following mild regularity condition on the coefficients of $g$.
\begin{assumption}\label{assume:coeffs}
    Denote
$
g(z)=\sum_{0\leq i\leq q} g_iz^i
$. We assume $\sup_i\abs{g_i} = \cO(1)$.
\end{assumption}

\begin{figure}
    \centering
    \includegraphics[width=0.5\textwidth]{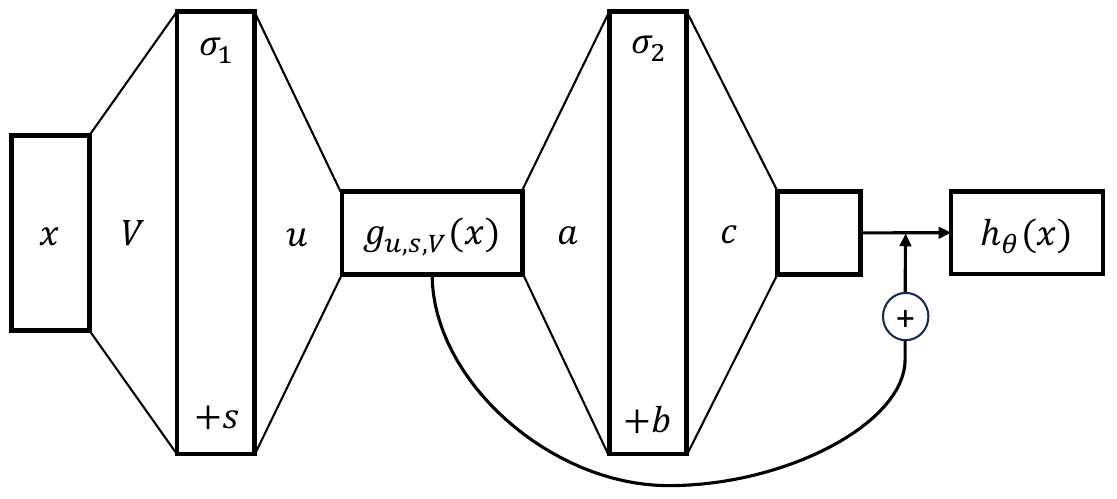}
    \caption{Three-layer network with bottleneck layer and residual link, defined in \eqref{eq:three-layer-nn}.}
    \label{fig:arch}
\end{figure}

\paragraph{Three Layer Network.}


Our learner is a three-layer neural network with a bottleneck layer and residual link. Let $m_1,m_2$ be the two hidden layer widths, and $\sigma_1(\cdot),\sigma_2(\cdot)$ be two activation functions. The network, denoted by $h_\theta$, is formally defined as follows:
\begin{equation}\label{eq:three-layer-nn}
\begin{aligned}
h_{\theta}(x):=g_{u,s,V}(x)&+c^{\top}\sigma_2(ag_{u,s,V}(x)+b)=
g_{u,s,V}(x)+\sum_{i=1}^{m_2} c_i\sigma_2(a_ig_{u,s,V}(x)+b_i) \\
g_{u,s,V}(x)&:=u^{\top}\sigma_1(Vx+s)
\end{aligned}
\end{equation}
where $a,b,c\in\RR^{m_2}$, $u,s\in \RR^{m_1}$ and $V\in \RR^{m_1\times d}$. Here, the intermediate embedding $g_{u,s, V}$ is a two-layer neural network with input $x$ and width $m_1$,  while the mapping $g_{u,s, V} \mapsto h_\theta$ is another two-layer neural network with input dimension $1$, width $m_2$, and a residual connection. We let $\theta:= (a, b, c, u,s, V)$ be an aggregation of all the parameters. We remark that the bottleneck layer and residual connection are similar to those in the ResNet architecture~\citep{he2016deep}, as well as architectures considered in prior theoretical work~\citep{ren2023depth, allen2019can, allen2020backward}. See \Cref{fig:arch} for a diagram of the network architecture.

The parameters $\theta^{(0)}:=(a^{(0)},b^{(0)},c^{(0)},u^{(0)},s^{(0)},V^{(0)})$ are initialized as $c^{(0)}=0$, $u^{(0)}=0$, $a_i^{(0)}\sim_{iid} \operatorname{Unif}\{-1,1\}$, $s_i^{(0)}\sim_{iid} \cN(0,1/2)$, and $v_i^{(0)} \sim_{iid} \operatorname{Unif}\{\mathbb{S}^{d-1}(1/\sqrt{2})\}$, the sphere of radius $1/\sqrt{2}$, where $\{v_i^{(0)}\}_{i \in [m_1]}$ are the rows of $V^{(0)}$.
Furthermore, we will assume $b^{(0)}_i\sim_{iid} \tau_b$, where $\tau_b$ is a distribution with density $\mu_b(\cdot)$. We make the following assumption on $\mu_b$:
\begin{assumption}
    $\mu_b(t) \gtrsim (\abs{t}+1)^{-p}$ for an absolute constant $p >0$, and $\EE_{b \sim \mu_b}[ b^8] \lesssim 1$.
\end{assumption}
\begin{remark}
For example, we can choose $\tau_b$ to be the Student's $t$-distribution with a degree of freedom larger than 8.
Student's $t$-distribution has the probability density function (PDF) given by
$$
\mu_\nu(t)=\frac{\Gamma\left(\frac{\nu+1}{2}\right)}{\sqrt{\nu \pi} \Gamma\left(\frac{\nu}{2}\right)}\left(1+\frac{t^2}{\nu}\right)^{-(\nu+1) / 2}
$$
where $\nu$ is the number of degrees of freedom and $\Gamma$ is the gamma function. 
\end{remark}

\paragraph{Training Algorithm.}
The network \eqref{eq:three-layer-nn} is trained via layer-wise gradient descent with sample splitting.
We generate two independent datasets $\cD_1,\cD_2$, each of which has $n$ independent samples $(x, h(x))$ with $x \sim \gamma$. We denote $\hat{L}_{\cD_i}(\theta)$ as the empirical square loss on $\cD_i$, i.e
\[
\hat{L}_{\cD_i}(\theta):=\frac{1}{n}\sum_{x\in \cD_i} \left(h_{\theta}(x)-h(x)\right)^2.
\]
In our training algorithm, we first train $u$ via gradient descent for $T_1$ steps on the empirical loss $\hat{L}_{\cD_1}(\theta)$, then train $c$ via gradient descent for $T_2$ steps on $\hat{L}_{\cD_2}(\theta)$. In the whole training process, $a,b,s,V$ are held fixed. The pseudocode for this training procedure is presented in Algorithm \ref{alg:layerwise}. 


\begin{algorithm}
\caption{Layer-wise Training Algorithm}\label{alg:layerwise}
\textbf{Input:} Initialization $\theta^{(0)}$, learning rate $\eta_1,\eta_2$, weight decay $\xi_1,\xi_2$, time $T_1,T_2$.
\begin{algorithmic}
    \For{$t=1,\dots,T_1$}
    \State{$u^{(t)}\leftarrow u^{(t-1)}-\eta_1(\nabla_{u}\hat{L}_{\cD_1}(\theta^{(t-1)})+\xi_1 u^{(t-1)})$}
    \State{$\theta^{(t)}\leftarrow (a^{(0)},b^{(0)},c^{(0)},u^{(t)},s^{(0)},V^{(0)})$}
    \EndFor
    \For{$t=T_1+1,\dots,T_1+T_2$}
\State{$c^{(t)}\leftarrow c^{(t-1)}-\eta_2(\nabla_{c}\hat{L}_{\cD_2}(\theta^{(t-1)})+\xi_2 c^{(t-1)})$}
\State{$\theta^{(t)}\leftarrow (a^{(0)},b^{(0)},c^{(t)},u^{(T_1)},s^{(0)},V^{(0)})$}
    \EndFor
    \State{$\hat{\theta}\leftarrow \theta^{(T_1+T_2)}$}
\end{algorithmic}
\textbf{Output:} $\hat{\theta}$.
\end{algorithm}

\subsection{Hermite Polynomials}
Our main results depend on the definition of the Hermite polynomials. We briefly introduce key properties of the Hermite polynomials here, and defer further details to Appendix \ref{app: hermite polynomial}.
\begin{definition}[1D Hermite polynomials] The $k$-th normalized probabilist's Hermite polynomial, $h_k: \RR \rightarrow \RR$, is the degree $k$ polynomial defined as
\begin{align}
    h_k(x) = \frac{(-1)^k}{\sqrt{k!}}\frac{\frac{d^k\mu_\beta}{dx^k}(x)}{\mu_\beta(x)},
\end{align}
where $\mu_\beta(x) = \exp(-x^2/2)/\sqrt{2\pi}$ is the density of the standard Gaussian.
\end{definition}
The first such Hermite polynomials are $$
h_0(z)=1, h_1(z)=z, h_2(z)=\frac{z^2-1}{\sqrt{2}}, h_3(z)=\frac{z^3-3 z}{\sqrt{6}}, \cdots
$$
Denote $\beta=\mathcal{N}(0,1)$ to be the standard Gaussian in 1D. A key fact is that the normalized Hermite polynomials form an orthonormal basis of $L^2(\beta)$; that is $\EE_{x \sim \beta}[h_j(x)h_k(x)] = \delta_{jk}$.

The multidimensional analogs of the Hermite polynomials are \emph{Hermite tensors}:
\begin{definition}[Hermite tensors] The $k$-th Hermite tensor in dimension $d$, $He_k: \RR^d \rightarrow (\RR^{d})^{\otimes k}$, is defined as
\begin{align*}
    He_k(x) := \frac{(-1)^k}{\sqrt{k!}}\frac{\nabla^k \mu_\gamma(x)}{\mu_\gamma(x)},
\end{align*}
where $\mu_\gamma(x) = \exp(-\frac12\|x\|^2)/(2\pi)^{d/2}$ is the density of the $d$-dimensional standard Gaussian.
\end{definition}
The Hermite tensors form an orthonormal basis of $L^2(\gamma)$; that is, for any $f \in L^2(\gamma)$, one can write the Hermite expansion
\begin{align*}
    f(x) = \sum_{k \geq 0}\langle C_k(f), He_k(x)\rangle \quad \text{where} \quad C_k(f) := \EE_{x \sim \gamma}[f(x)He_k(x)] \in (\RR^d)^{\otimes k}.
\end{align*}
As such, for any integer $k \geq 0$ we can define the projection operator $\cP_k : L^2(\gamma) \rightarrow L^2(\gamma)$ onto the span of degree $k$ Hermite polynomials as follows:
\begin{align*}
    (\cP_kf)(x) := \langle C_k(f), He_k(x) \rangle.
\end{align*}
Furthermore, denote $\cP_{\leq k} :=\sum_{0\leq i\leq k} \cP_i$ and $\cP_{<k}:=\sum_{0\leq i< k} \cP_i$ as the projection operators onto the span of Hermite polynomials with degree no more than $k$, and degree less than $k$, respectively.



\section{Main Results}
Our goal is to show that the network defined in \eqref{eq:three-layer-nn} trained via \Cref{alg:layerwise} can efficiently learn hierarchical polynomials of the form $h = g \circ p$.

First, we consider a restricted class of degree $k$ polynomials for the hidden feature $p$. Consider $p$ with the following decomposition:
\begin{equation}\label{eq:feature}
p(x)=\frac{1}{\sqrt{L}}\left(\sum_{i=1}^L \lambda_i \psi_i(x)\right).
\end{equation}
\begin{restatable}{assumption}{feature}
\label{assumption: on feature p in main text} The feature $p$ can be written in the form \eqref{eq:feature}. We make the following additional assumptions on $p$:
\begin{itemize}
\item There is a set of orthogonal vectors $\{v_{i,j}\}_{i\in [L],j\in [J_i]}$, satisfying $J_i\leq k$ and $\|v_{i,j}\|=1$, such that $\psi_i(x)$ only depends on $v_{i,1}^{\top}x,\dots,v_{i,J_i}^{\top}x$.
\item For each $i$, $\cP_k\psi_i = \psi_i$. Equivalently, $\psi_i$ lies in the span of degree $k$ Hermite polynomials.
\item $\EE\left[\psi_i(x)^2\right]=1$ and $\EE\left[p(x)^2\right]=1$.
\item The $\lambda_i$ are balanced, i.e $\sup_i \abs{\lambda_i}=\cO(1)$, and $L = \Theta(d)$.
\end{itemize}
\end{restatable}
\begin{remark}
The first assumption tells us that each $\psi_i$ depends on a different rank $\leq k$ subspace, all of which are orthogonal to each other. As a consequence of the rotation invariance of the Gaussian, the quantities $\psi_i(x)$ are thus independent when we regard $x$ as a random vector. The second assumption requires $p$ to be a degree $k$ polynomial orthogonal to lower-degree polynomials, while the third is a normalization condition. The final condition requires $p$ to be sufficiently spread out, and depend on many $\psi_i$. Our results can easily be extended to any $L = \omega_d(1)$, at the expense of a worse error floor.
\end{remark}

\begin{remark}
Since $\cP_k \psi_i =\psi_i$ for each $i$, we have $\cP_k p =p$. We can thus write $p(x)$ as $\langle A, He_k(x)\rangle$ for some $A \in (\mathbb{R}^d)^{\otimes k}$. There are two important classes of $A$ which satisfy \Cref{assumption: on feature p in main text}:

   First, let $A$ be an orthogonally decomposable tensor
   \[
   A=\frac{1}{\sqrt{L}}\left(\sum_{i=1}^L \lambda_i v_i^{\otimes k}\right)
   \]
   where $\langle v_i,v_j \rangle = \delta_{ij}$. Using identities for the Hermite polynomials (\Cref{app: hermite polynomial}), one can rewrite the feature $p$ as
\begin{equation}\label{eq:feature_orthogonal}
p(x)=\frac{1}{\sqrt{L}}\left(\sum_{i=1}^L\lambda_i \langle v_i^{\otimes k}, He_k(x)\rangle\right)=\frac{1}{\sqrt{L}}\left(\sum_{i=1}^L \lambda_i h_k(v_i^{\top}x)\right).
\end{equation}
$p$ thus satisfies \Cref{assumption: on feature p in main text} with $J_i=1$ for all $i$, assuming the regularity conditions hold.

Next, we show that \Cref{assumption: on feature p in main text} is met when $p$ is a sum of sparse parities, i.e., \[
    A=\frac{1}{\sqrt{L}}\left(\sum_{i=1}^L \lambda_i\cdot v_{i,1}\otimes\dots\otimes v_{i,k}\right)
    \]
    where $\langle v_{i_1,j_1}, v_{i_2,j_2}\rangle = \delta_{i_1i_2}\delta_{j_1j_2}$. 
    In that case, the feature $p$ can be rewritten as
\[
p(x)=\frac{1}{\sqrt{L}}\left(\sum_{i=1}^L \lambda_i \langle v_{i,1}\otimes\dots\otimes v_{i,k},He_k(x)\rangle\right)=\frac{1}{\sqrt{L}}\left(\sum_{i=1}^L \lambda_i\left(\prod_{j=1}^k \langle v_{i,j},x\rangle\right)\right)
\]
For example, taking $L = d/k$ and choosing $v_{i,j}=e_{k(i-1)+j}$, the standard basis elements in $\RR^d$, the feature $p$ becomes\[p(x)=\frac{1}{\sqrt{d/k}}\left(\lambda_1x_1x_2\cdots x_k + x_{k+1}\cdots x_{2k} + \lambda_{d/k}x_{d-k+1}\cdots x_{d}\right)\]
and hence the name ``sum of sparse parities." This feature satisfies \Cref{assumption: on feature p in main text} with $J_i=k$ for all $i$, assuming that the regularity conditions hold.
\end{remark}

We next require the following mild assumptions on the link function $g$ and target $h$. The assumption on $h$ is purely for technical convenience and can be achieved by a simple pre-processing step. The assumption on $g$, in the single-index model literature~\citep{arous2021online}, is referred to as $g$ having an \emph{information exponent} of 1.
\begin{assumption}\label{assumption: zero expectation and information exponent}
    $\EE_{x\sim\gamma}\left[h(x)\right]=0$ and $\EE_{z\sim \cN(0,1)}\left[g'(z)\right]=\Theta(1)$.
\end{assumption}

Finally, we make the following assumption on the activation functions $\sigma_1, \sigma_2$:
\begin{assumption}\label{assumption: activation function}
    We assume $\sigma_1$ is a $k$ degree polynomial. Denote $\sigma_1(z)=\sum_{0\leq i\leq k}o_iz^i$, we further assume $\sup_i\abs{o_i}=\cO(1)$ and $\abs{o_k}=\Theta(1)$.
    Also, set $\sigma_2(z)=\operatorname{max}\{z,0\}$, i.e., the $\operatorname{ReLU}$ activation.
\end{assumption}

With our assumptions in place, we are ready to state our main theorem.

\begin{theorem}\label{thm:main_thm}
Under the above assumptions, for any constant $\alpha\in (0,1)$, any $m_1\geq d^{k+\alpha}$ and any $n\geq d^{k+3\alpha}$, set $m_2=d^{\alpha}$, $T_1=\operatorname{poly}(d,m_1,n)$, $T_2=\operatorname{poly}(d,m_1,m_2,n)$, $\eta_1=\frac{1}{\operatorname{poly}(d,m_1,n)}$, $\eta_2=\frac{1}{\operatorname{poly}(d,m_1,m_2,n)}$, $\xi_1=\frac{2m_1}{d^{k+\alpha}}$ and $\xi_2=2$. Then, for any absolute constant $\delta\in (0,1)$, with probability at least $1-\delta$ over the sampling of initialization and the sampling of training dataset $\cD_1,\cD_2$, the estimator $\hat{\theta}$ output by \Cref{alg:layerwise} satisfies 
\[
\norm{h_{\hat{\theta}}-h}{L^2(\gamma)}^2 = \widetilde{\cO}(d^{-\alpha}).
\]
\end{theorem}

\Cref{thm:main_thm} states that \Cref{alg:layerwise} can learn the target $h=g\circ p$ in $n = \widetilde \cO(d^k)$ samples,
with widths $m_1 = \widetilde \Theta(d^k), m_2 = \widetilde \Theta(1)$. Up to log factors, this is the same sample complexity as directly learning the feature $p$.
On the other hand, kernel methods such as the NTK require $n = \widetilde \Omega(d^{kq})$ samples to learn $h$, and are unable to take advantage of the underlying hierarchical structure. 

A simple corollary of \Cref{thm:main_thm} follows when $k = 2$. In this case the feature $p$ is a quadratic polynomial and can be expressed as the following for some symmetric $A \in \RR^{d \times d}$
\begin{align*}
    p(x) = \langle A, xx^{\top} - I \rangle = x^{\top}Ax - \tr(A).
\end{align*}
Taking $\tr(A) = 0$, and noting that since $A$ always has an eigendecomposition, \Cref{assumption: on feature p in main text} is equivalent to $\norm{A}{F} = 1$ and $\norm{A}{op} = \cO(1/\sqrt{d})$, one obtains the following:
\begin{corollary}\label{cor:quad}
    Let $h(x) = g(x^{\top}Ax)$ where $\tr(A) = 0, \norm{A}{F} = 1$, and $\norm{A}{op} = \cO(1/\sqrt{d})$. Then under the same setting of hyperparameters as \Cref{thm:main_thm}, for any sample size $n \geq d^{2 + 3\alpha}$, with probability at least $1 - \delta$ over the initialization and data, the estimator $\hat \theta$ satisfies
   \[
\norm{h_{\hat{\theta}}-h}{L^2(\gamma)}^2 = \widetilde{\cO}(d^{-\alpha}).
\]
\end{corollary}

\Cref{cor:quad} states that \Cref{alg:layerwise} can learn the target $g(x^{\top}Ax)$ in $\widetilde \cO(d^2)$ samples, which matches the information-theoretically optimal sample complexity. This improves over the sample complexity of the algorithm in \citet{nichani2023provable} when $g$ is a polynomial, which requires $\widetilde\Theta(d^4)$ samples. See \Cref{sec:sample_complexity_improvement} for discussion on why \Cref{alg:layerwise} is able to obtain this sample complexity improvement.

\section{Proof Sketch}

The proof of \Cref{thm:main_thm} proceeds by analyzing each of the two stages of training. First, we show that after the first stage, the network learns to {extract} the hidden feature $p$ out (\Cref{sec:stage1}). Next, we show that during the second stage, the network learns the link function $g$ (\Cref{sec:stage2}).

\subsection{Stage 1: Feature Learning}
\label{sec:stage1}

The first stage of training is the feature learning stage. Here, the network learns to extract the degree $k$ polynomial feature so that the intermediate layer satisfies $g_{u, s,V} \approx p$ (up to a scaling constant).

At initialization, the network satisfies $h_\theta = g_{u, s,V}$. Thus during the first stage of training, the network trains $u$ to fit $g_{u,s, V}$ to the target $h$. Since the activation $\sigma_1$ is a degree $k$ polynomial with $o_k=\Theta(1)$, we can indeed prove that at the end of the first stage $g_{u, s,V}$ will learn to fit the best degree $k$ polynomial approximation to $h$, $\cP_{\leq k} h$ (Lemma \ref{lemma: only capture k degree}). During the first stage the loss is convex in $u$, and thus optimization and generalization can be handled via straightforward kernel arguments. The following lemma formalizes the above argument, and shows that at the end of the first stage the network learns to approximate $\cP_{\leq k} h$.
\begin{restatable}{lemma}{stageone}
\label{thm:kernel_stage1}
For any constant $\alpha \in (0,1)$, any $m_1\geq d^{k+\alpha}$ and any $n\geq d^{k+3\alpha}$, set $T_1 =\operatorname{poly}(n,m_1,d)$, $\eta_1=\frac{1}{\operatorname{poly}(n,m_1,d)}$ and  $\xi_1=\frac{2m_1}{d^{k+\alpha}}$. Then, for any absolute constant $\delta\in (0,1)$,
with probability at least $1-\delta/2$ over the initialization $V,s$ and training data $\cD_1$, we have
\[
\norm{h_{\theta^{(T_1)}}-\cP_{\leq k} h}{L^2(\gamma)}^2 = \widetilde{\cO}(d^{-\alpha}).
\]
\end{restatable}

It thus suffices to analyze the quantity $\cP_{\leq k} h$.
Our key technical result, and a main innovation of our paper, is \Cref{lem:approximate_stein}. It shows that the term $\cP_{\leq k} h$ is approximately equal to $\cP_k h$, and furthermore, up to a scaling constant, $\cP_kh$ is approximately equal to the hidden feature $p$:

\begin{lemma}\label{lem:approximate_stein}Under the previous assumptions, we have
\[
\norm{\cP_k h-\EE_{z\sim \cN(0,1)}\left[g'(z)\right] p}{L^2(\gamma)} = \cO(d^{-1/2})\quad \text{and}\quad \norm{\cP_{<k} h}{L^2(\gamma)}=\cO(d^{-1/2})
\]
\end{lemma}

A proof sketch of \Cref{lem:approximate_stein} is deferred to \Cref{sec:approximate_stein}, with the full proof in \Cref{app:hermite_projection}.

 Combining \Cref{thm:kernel_stage1} and \Cref{lem:approximate_stein}, we obtain the performance after the first stage: 

\begin{corollary}\label{lem:stage_1_combined}
    Under the setting of hyperparameters in \Cref{thm:main_thm}, for any constants $\alpha, \delta \in (0,1)$, with probability $1 - \delta/2$ over the initialization and the data $\cD_1$, the network after time $T_1$ satisfies
    \begin{align*}
        \norm{h_{\theta^{(T_1)}}-\EE_{z\sim \cN(0,1)}\left[g'(z)\right] p}{L^2(\gamma)}^2 = \widetilde{\cO}(d^{-\alpha}).
    \end{align*}
\end{corollary}
Proofs for stage 1 are deferred to \Cref{sec:stage1_proofs}.

\subsection{Stage 2: Learning the Link Function}
\label{sec:stage2}

After the first stage of training, $g_{u, s,V}$ is approximately equal to the true feature $p$ up to a scaling constant. The second stage of training uses this feature to learn the link function $g$. Specifically, the second stage aims to fit the function $g$ using the two-layer network $z \mapsto z + c^{\top}\sigma_2(az+b)$. Since only $c$ is trained during stage 2, the network is a random feature model and the loss is convex in $c$.

Our main lemma for stage 2 shows that there exists $c^*$ with low norm such that the parameter vector $\theta^* := (a^{(0)}, b^{(0)}, c^*, u^{(T_1)}, s^{(0)}, V^{(0)})$ satisfies $h_{\theta^*} \approx h$. Let $\hat p$ be an arbitrary degree $k$ polynomial satisfying $\norm{\hat p-\EE_{z\sim \cN(0,1)}\left[g'(z)\right] p}{L^2(\gamma)}^2 = \cO((\log d)^{r/2}d^{-\alpha})$ (and recall that after stage 1, $g_{u,s,V}$ satisfies this condition with high probability). The main lemma is the following.
\begin{restatable}{lemma}{stagetwo}
\label{lem:stage_2}
Let $m = d^\alpha$. With probability at least $1-\delta/4$ over the sampling of $a,b$, there exists some $c^*$ such that $\|c^*\|_{\infty}=\cO((\log d)^{k(p+q)}d^{-\alpha})$ and
\[
L(\theta^*)=\norm{\hat{p}(x)+\sum_{i=1}^m c_i^*\sigma(a_i\hat{p}(x)+b_i)-h(x)}{L^2(\gamma)}^2  = \cO((\log d)^{r/2+2k(p+q)}d^{-\alpha})
\]
\end{restatable}

Since the regularized loss is strongly convex in $c$, GD converges linearly to some $\hat \theta$ with $\hat L_2(\hat \theta) \lesssim \hat L_2(\theta^*)$ and $\norm{\hat c}{2} \lesssim \norm{c^*}{2}$. Finally, we invoke standard kernel Rademacher arguments to show that, since the link function $g$ is one-dimensional, $n = \widetilde \cO(1)$ sample suffice for generalization in this stage. Combining everything yields \Cref{thm:main_thm}. Proofs for stage 2 are deferred to \Cref{sec:stage2_proofs}.
\subsection{The Approximate Stein's Lemma}\label{sec:approximate_stein}


To conclude the full proof of \Cref{thm:main_thm}, it suffices to prove \Cref{lem:approximate_stein}. \Cref{lem:approximate_stein} can be interpreted as an approximate version of Stein's lemma, generalizing the result in \citet{nichani2023provable} to polynomials of degree $k > 2$. To understand this intuition, we first recall Stein's lemma:
\begin{lemma}[Stein's Lemma]
For any $g: \RR \rightarrow \RR$ and $g\in C^1$, one has
\begin{align*}
    \EE_{z \sim \cN(0, 1)}[z g(z)] = \EE_{z \sim \cN(0, 1)}[g'(z)].
\end{align*}
\end{lemma}
Recall that the feature is of the form $p(x) = \frac{1}{\sqrt{L}}\sum_{i=1}^L \lambda_i \psi_i(x)$. Since each $\psi_i$ depends only on the projection of $x$ onto $\{v_{i, 1}, \dots, v_{i, J_i}\}$, and these vectors are orthonormal, the individual terms $\psi_i(x)$ are independent random variables. Furthermore they satisfy $\EE[\psi_i(x)] = 0$ and $\EE[\psi_i(x)^2] = 1$. Since $L = \Theta(d)$, the Central Limit Theorem tells us that in the $d \rightarrow \infty$ limit
\begin{align*}
     \frac{1}{\sqrt{L}}\sum_{i=1}^L \lambda_i \psi_i \rightarrow_d \cN(0, 1)
\end{align*}
when the $\lambda_i$ are balanced. The distribution of the feature $p$ is thus ``close" to a Gaussian. As a consequence, one expects that
\begin{align}\label{eq:stein_p}
    \EE_{x \sim \gamma}[p(x)g(p(x))] \approx \EE_{z \sim \cN(0, 1)}[z g(z)] = \EE_{z \sim \cN(0, 1)}[g'(z)].
\end{align}
Next, let $q$ be another degree $k$ polynomial such that $\norm{q}{L^2(\gamma)} = 1$ and $\langle p, q \rangle_{L^2(\gamma)} = 0$. For most $q$, we can expect that $(p, q)$ is approximately jointly Gaussian. In this case, $p$ and $q$ are approximately independent due to $\langle p, q \rangle_{L^2(\gamma)} = 0$, and as a consequence
\begin{align}\label{eq:stein_q}
    \EE_{x \sim \gamma}[q(x)g(p(x))] \approx \EE_{x \sim \gamma}[q(x)]\EE_{x \sim \gamma}[g(p(x))] = 0.
\end{align}
\eqref{eq:stein_p} and \eqref{eq:stein_q} imply that the degree $k$ polynomial $g\circ p$ has maximum correlation with is $p$, and thus
\[
\cP_k(g\circ p) \approx \EE_{z \sim \cN(0, 1)}[g'(z)]p.
\]
Similarly, if $q$ is a degree $<k$ polynomial, then since $\mathcal{P}_kp = p$ one has $\langle p, q\rangle_{L^2(\gamma)} = 0$. Again, we can expect that $p, q$ are approximately independent, which implies that $\langle h, q \rangle_{L^2(\gamma)} \approx 0$.

We remark that the preceding heuristic argument, and in particular the claim that $p$ and $q$ are approximately independent, is simply to provide intuition for \Cref{lem:approximate_stein}. The full proof of \Cref{lem:approximate_stein}, provided in \Cref{app:hermite_projection}, proceeds by expanding the polynomial $g\circ p$ into sums of products of monomials, and carefully analyzes the degree $k$ projection of each of the terms. 

\section{Experiments}\label{app:experiments}
We empirically verify \Cref{thm:main_thm}, and demonstrate that three-layer neural networks indeed learn hierarchical polynomials $g \circ p$ by learning to extract the feature $p$. 

Our experimental setup is as follows. The target feature is of the form $h=g\circ p$, $p(x) = \sum_{i=1}^d \lambda_i h_3(x_i)$, where the $\lambda_i$ are drawn i.i.d from $\left\{\pm \frac{1}{\sqrt{d}}\right\}$ uniformly, and the link function is $g(z) = C_dz^3$, where $C_d$ is a normalizing constant chosen so $\EE_x[h(x)^2] = 1$. Our architecture is the same ResNet-like architecture defined in \eqref{eq:three-layer-nn}, with activations $\sigma_1(z) = z^3$ and $\sigma_2 = \operatorname{ReLU}$. We additionally use the $\mu$P initialization~\citep{yang2021tensor}. For a chosen input dimension $d$ and sample size $n$, we choose hidden layer widths $m_1 = d^2$ and $m_2 = 1000$. We optimize the empirical square loss to convergence by simultaneously training all parameters $(u, s, V, a, b, c)$ using the Adam optimizer. We then compute the test loss of the learned predictor, as well as the correlation between the ``learned feature" (defined to be $g_{u,s,V}$) and the ``true feature" $p$ on these test points. 

In~\Cref{fig:experiments}, we plot both the test loss and feature correlation as a function of $n$, for $d \in \{16, 24, 32, 40\}$. We observe that, across varying values of depth, roughly $d^3$ samples are needed to learn $h$ up to near zero test error. Additionally, we observe that as $n$ grows past $d^3$, the correlation between the true feature and learned feature approaches $1$. This demonstrates that the network is indeed performing feature learning, and learns to fit $p$ using $g_{u, s, V}$ in order to learn the entire function. Overall, this demonstrates that our high-level insight that the sample complexity of learning $g \circ p$ is equal to the sample complexity of $p$, and that three-layer neural networks implement the more efficient algorithm of learning to first extract $p$ out of $g\circ p$, holds in the more realistic setting where all parameters of the network are trained jointly.

\begin{figure}
    \centering
    \includegraphics[width=0.47\textwidth]{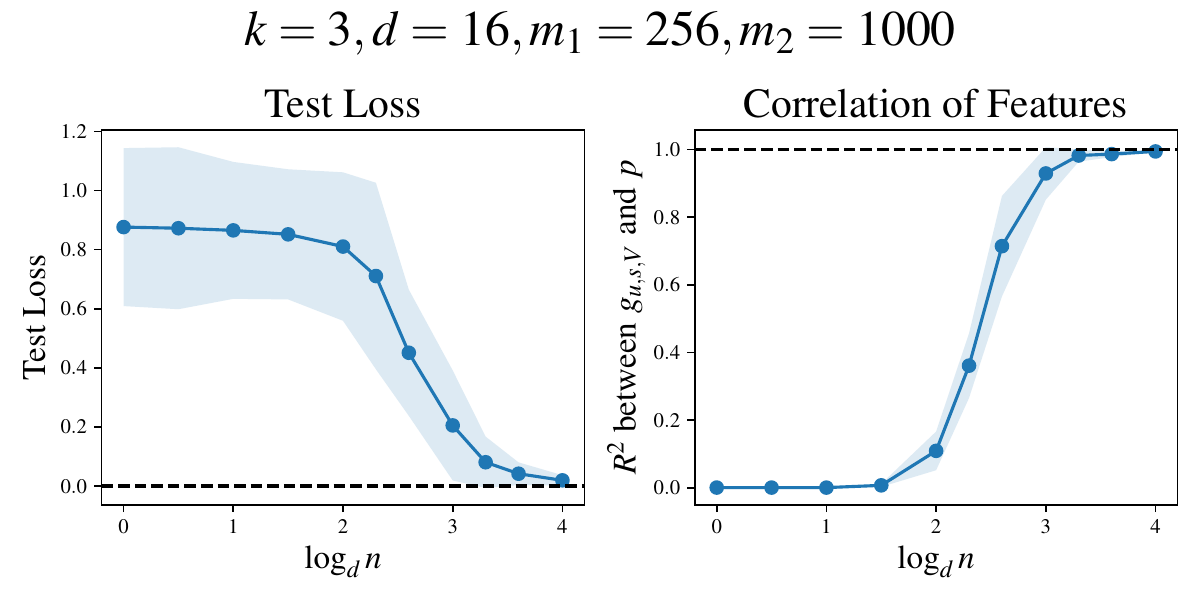}
    \includegraphics[width=0.47\textwidth]{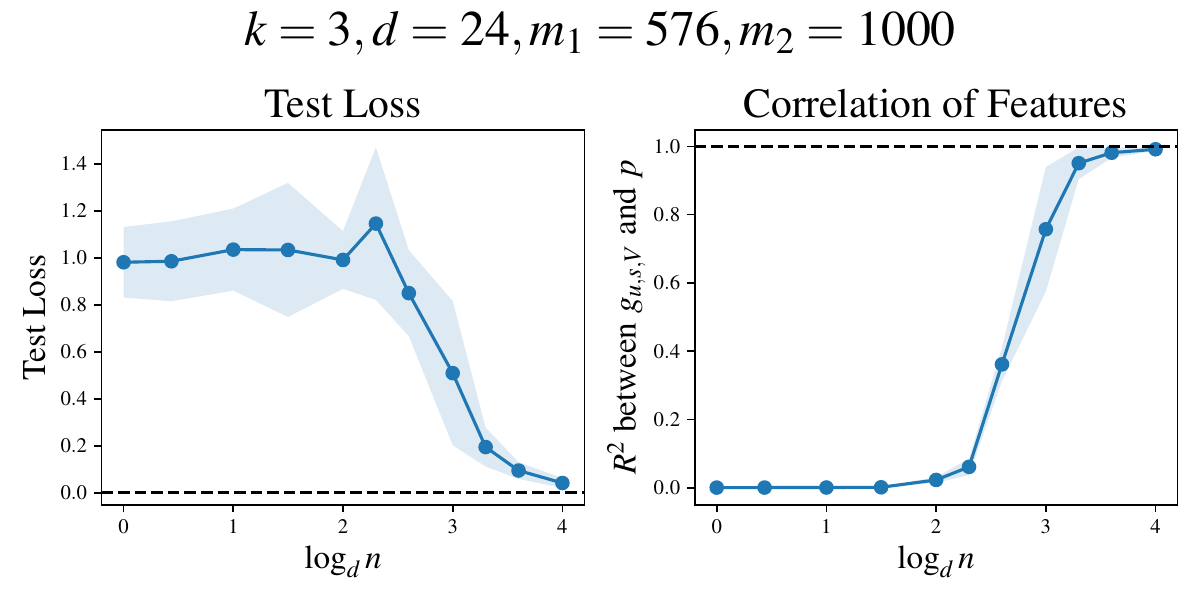}
    \includegraphics[width=0.47\textwidth]{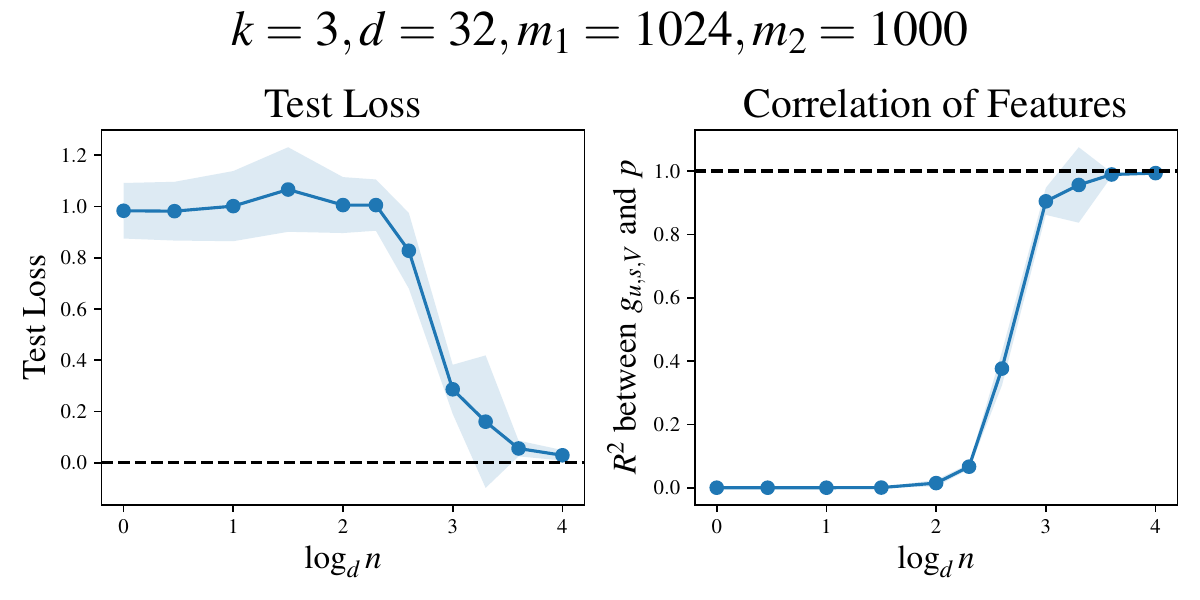}
    \includegraphics[width=0.47\textwidth]{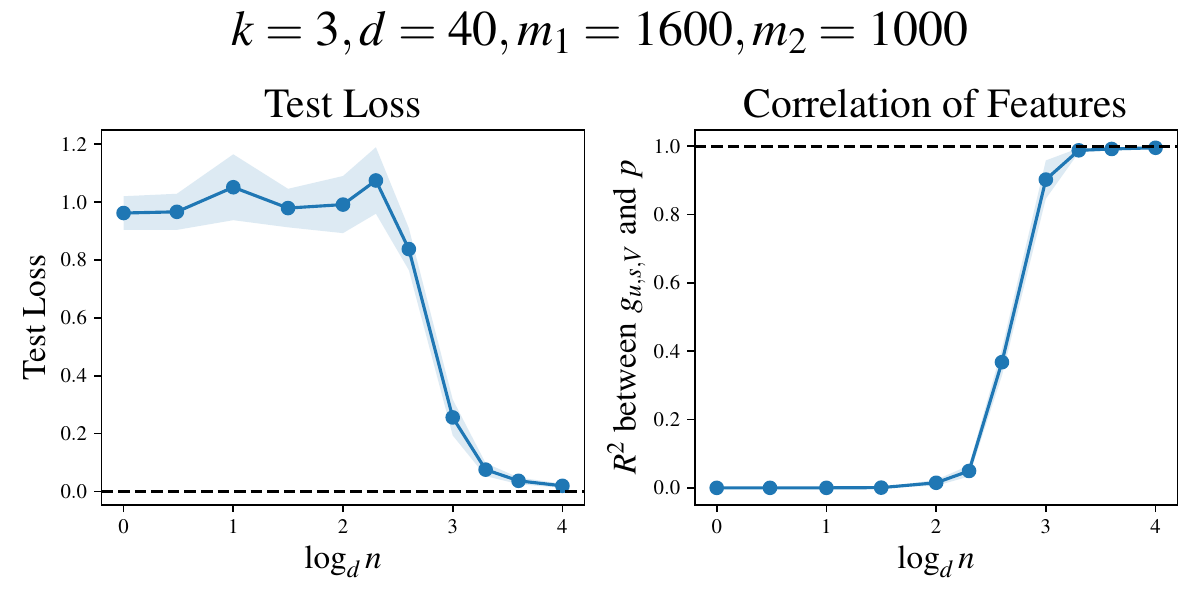}
    \caption{We train the ResNet architecture \eqref{eq:three-layer-nn} to learn the hierarchical polynomial $h = g \circ p$ when the degree of $p$ is $k=3$. We observe that the network learns the true feature $p$, as measured by the correlation between $g_{u, s, V}$ and $p$ (right panel of each figure). As a consequence, the network can learn $h$ in $d^3$ samples (left panel of each figure).}
    \label{fig:experiments}
\end{figure}

\textbf{Experimental Details.} Our experiments were written in JAX~\citep{jax2018github} and run on a single
NVIDIA RTX A6000 GPU.
\section{Discussion}

\subsection{Comparison to \citet{nichani2023provable}}\label{sec:sample_complexity_improvement}

In the case where $k = 2$ and the feature is a quadratic, \Cref{cor:quad} tells us that \Cref{alg:layerwise} requires $\widetilde\cO(d^2)$ samples to learn $h$, which matches the information-theoretic lower bound. This is an improvement over \citet{nichani2023provable}, which requires $\widetilde\Theta(d^4)$ samples. 

The key to this sample complexity improvement is that our algorithm runs GD for \emph{many steps} during the first stage to completely extract the feature $p(x)$, whereas the first stage in \citet{nichani2023provable} takes a single large gradient step, which can only weakly recover the true feature. Specifically, \citet{nichani2023provable} considers three-layer neural networks of the form $h_\theta(x) = a^{\top}\sigma_2(W\sigma_1(Vx)+b)$, and shows that after the first large step of GD on the population loss, the network satisfies $w_i^{\top} \sigma_1(Vx) \approx d^{-2}p(x).$ As a consequence, due to standard $1/\sqrt{n}$ concentration, $n = \widetilde\Omega(d^4)$ samples are needed to concentrate this term and recover the true feature. 

On the other hand, the first stage of \Cref{alg:layerwise} directly fits the best degree 2 polynomial to the target. It thus suffices to uniformly concentrate the loss landscape, which only requires $\widetilde\cO(d^2)$ samples as the learner is fitting a quadratic. Running GD for many steps is thus key to obtaining this optimal sample complexity. We remark that \citet{nichani2023provable} handles a slightly larger class of link functions $g$ (1-Lipschitz functions) and activations $\sigma_1$ (nonzero second Hermite coefficient).


\subsection{Layerwise Gradient Descent on Three-Layer Networks}

\Cref{alg:layerwise} takes advantage of the underlying hierarchical structure in $h$ to learn in $\widetilde\Theta(d^k)$ samples. Regular kernel methods, however, cannot utilize this hierarchical structure, and thus require $\widetilde\Theta(d^{kq})$ samples to learn $h$ up to vanishing error. Each stage of \Cref{alg:layerwise} implements a kernel method: stage 1 uses kernel regression to learn $p$ in $\widetilde\cO(d^k)$ samples, while stage 2 uses kernel regression to learn $g$ in $\widetilde\cO(1)$ samples. Crucially, however, \emph{our overall algorithm is not a kernel method}, and can learn hierarchical functions with a significantly improved sample complexity over naively using a single kernel method to learn the entire function. It is a fascinating question to understand which other tasks can be learned more efficiently via such layerwise GD. While \Cref{alg:layerwise} is layerwise, and thus amenable to analysis, it still reflects the ability of three-layer networks in practice to learn hierarchical targets; see \Cref{app:experiments} for experiments with more standard training procedures.


\subsection{Future Work}
In this work, we showed that three-layer neural networks are able to efficiently learn hierarchical polynomials of the form $h = g \circ p$, for a large class of degree $k$ polynomials $p$. An interesting direction is to understand whether our results can be generalized to \emph{all} degree $k$ polynomials. We conjecture that our results should still hold as long as $p$ is homogeneous and close in distribution to a Gaussian, which should be true for more general tensors $A$. Additionally, the target functions we consider depend on only a single hidden feature $p$. It is interesting to understand whether deep networks can efficiently learn targets that depend on multiple features, i.e. of the form $h(x) = g(p_1(x), \dots, p_R(x))$ for some $g: \RR^R \rightarrow \RR$.



\clearpage

\bibliography{ref}
\clearpage
\appendix
\begin{center}
    \noindent\rule{\textwidth}{4pt} \vspace{-0.2cm}
    
    \LARGE \textbf{Appendix} 
    
    \noindent\rule{\textwidth}{1.2pt}
\end{center}

\startcontents[sections]
\printcontents[sections]{l}{1}{\setcounter{tocdepth}{2}}

\section{Proof of \Cref{lem:approximate_stein}}\label{app:hermite_projection}

\subsection{Results for General Features}

In this subsection, we will consider the following feature class
\[
p(x)=\frac{1}{\sqrt{L}}\left(\sum_{i=1}^L \lambda_i \psi_i(x)\right)
\]
Recall our assumptions on $p$:
\feature*


Next, recall that the link function $g(z)=\sum_{0\leq i\leq q} g_iz^i$ satisfies $\sup_i \abs{g_i}=\cO(1)$ by \Cref{assume:coeffs}. Denote $h=g\circ p$. 
Due to Assumption \ref{assumption: zero expectation and information exponent}, we naturally have $\cP_0 h =\EE_{x\sim\gamma}\left[h(x)\right]=0$. Next,
we will prove the following two Lemmas, which directly implies \Cref{lem:approximate_stein}.

\begin{lemma}\label{lemma appendix A first projection}

Under all the assumptions above, we have \[
\norm{\cP_k h-\EE_{z\sim \cN(0,1)}\left[g'(z)\right] p}{L^2(\gamma)} = \cO(L^{-1/2})
\]
\end{lemma}
\begin{lemma}\label{lemma: appendix A second projection}
Under all the assumptions above, for any $1\leq m \leq k-1 $ we have
\[
\norm{\cP_m h}{L^2(\gamma)} = \cO(L^{-1/2})
\]

\end{lemma}
\begin{proof}[Proof of Lemma \ref{lemma appendix A first projection}]

Firstly, we will compute the Hermite degree $k$ components of $p(x)^w$, $w\geq 2$. From the definition of $\cP_k$ and multinomial expansion theorem, we know 
\begin{equation}\label{equation: expanding}
\begin{aligned}
    \cP_k (p(x)^w)&=\frac{1}{L^{w/2}}\left(\sum_{i=1}^L \lambda_i \psi_i(x)\cP_0\left(\sum_{z_i\geq 2, q, z_1+\dots+z_q=w-1,i_j\neq i}\frac{w!}{z_1!\dots z_q!}\lambda_{i_1}^{z_1}\dots\lambda_{i_q}^{z_q}\psi_{i_1}(x)^{z_1}\dots \psi_{i_q}(x)^{z_q}\right)\right)\\
    &\quad+\frac{1}{L^{w/2}}\cP_k\left(\sum_{z_i\geq 2,q,z_1+\dots+z_q=w,i_j}\frac{w!}{z_1!\dots z_q!}\lambda_{i_1}^{z_1}\dots\lambda_{i_q}^{z_q}\psi_{i_1}(x)^{z_1}\dots \psi_{i_q}(x)^{z_q}\right)
    \end{aligned}
\end{equation}
by expanding $\left(\frac{1}{\sqrt{L}}\left(\sum_{1\leq i \leq L} \lambda_i \psi_i(x)\right)\right)^w$ and computing the projection for each term.
The key observation that leads to \eqref{equation: expanding} is the following: 
\begin{lemma}\label{lem:hermite_product}
    Let $\phi_1, \phi_2 \in L^2(\gamma)$ be two functions such that $\phi_1$ lies in the span of degree $k_1$ Hermite polynomials and $\phi_2$ lies in the span of degree $k_2$ Hermite polynomials. That is to say, $\cP_{k_i}\phi_i = \phi_i$ for $i=1,2$.
    
    If $\phi_1, \phi_2$ only depend on the projection of $x$ onto subspaces $V_1, V_2$ respectively, and $V_1, V_2$ are orthogonal to each other, i.e $V_1V_2^{\top} = 0$, then $\cP_{k_1 + k_2}(\phi_1\phi_2) = \phi_1\phi_2$.
\end{lemma}
\Cref{lem:hermite_product} follows directly from the fact that the $d$-dimensional Hermite basis is formed from taking products of the $1$-dimensional Hermite basis elements.

In the above expansion, if there are two indices $i_1,i_2$ each with exponent 1, then we get a $\psi_{i_1}(x)\psi_{i_2}(x)\prod_{j\geq 3} \psi_{{i_j}}(x)^{z_j}$ term. By \Cref{lem:hermite_product}, this term is a polynomial with Hermite degree at least $2k$. Equivalently\[
\cP_k\left( \psi_{i_1}(x)\psi_{i_2}(x)\prod_{j\geq 3} \psi_{{i_j}}(x)^{z_j}\right)=0.
\] This is because $\psi_i(x)$ only depends on $v_{i,1}^{\top}x,\dots,v_{i,J_i}^{\top}x$ and $\{v_{i,j}\}_{i\in [L],j\in [J_i]}$ are orthogonal vectors. Similarly, for terms of the form $\psi_{i_1}(x)\prod_{j\geq 2} \psi_{{i_j}}(x)^{z_j}$, we have that
\[
\cP_k\left(\psi_{i_1}(x)\prod_{j\geq 2} \psi_{{i_j}}(x)^{z_j}\right) = \psi_{i_1}(x)\cP_0\left(\prod_{j\geq 2} \psi_{{i_j}}(x)^{z_j}\right).
\]
Altogether, this gives \eqref{equation: expanding} above.



Let us firstly compute the $\cP_0$ terms in the above equation \eqref{equation: expanding}.

\paragraph{Case I.}Firstly consider the case that $w$ is odd and $w=2s+1$. Then we have
\[
\begin{aligned}
&\sum_{z_i\geq 2, q, z_1+\dots+z_q=w-1,i_j\neq i}\frac{w!}{z_1!\dots z_q!}\lambda_{i_1}^{z_1}\dots\lambda_{i_q}^{z_q}\psi_{i_1}(x)^{z_1}\dots \psi_{i_q}(x)^{z_q}
=\sum_{i_j\neq i}\frac{w!}{2^s}\lambda_{i_1}^{2}\dots\lambda_{i_s}^{2}\psi_{i_1}(x)^{2}\dots \psi_{{i_s}}(x)^{2}\\&\quad\quad\quad+\sum_{z_i\geq 2, q<s, z_1+\dots+z_q=w-1,i_j\neq i}\frac{w!}{z_1!\dots z_q!}\lambda_{i_1}^{z_1}\dots\lambda_{i_q}^{z_q}\psi_{i_1}(x)^{z_1}\dots \psi_{i_q}(x)^{z_q}
\end{aligned}
\]
For the first term, we have 
\begin{equation}\label{equation: first term first part}
\cP_0\left(\sum_{i_j\neq i}\frac{w!}{2^s}\lambda_{i_1}^{2}\dots\lambda_{i_s}^{2}\psi_{i_1}(x)^{2}\dots \phi_{S_{i_s}}(x)^{2}\right)=\sum_{i_j\neq i}\frac{w!}{2^s}\lambda_{i_1}^{2}\dots\lambda_{i_s}^{2}\EE\left[\psi_{i_1}(x)^{2}\dots \psi_{{i_s}}(x)^{2}\right]=\frac{w!}{2^s}\sum_{i_j\neq i} \lambda_{i_1}^2\dots\lambda_{i_s}^2
\end{equation}
For the second term, we count the number of monomials to get \begin{equation}\label{equa: new sparse parity term number small}
\begin{aligned}
&\abs{\cP_0\left(\sum_{z_i\geq 2, q<s, z_1+\dots+z_q=w-1,i_j\neq i}\frac{w!}{z_1!\dots z_q!}\lambda_{i_1}^{z_1}\dots\lambda_{i_q}^{z_q}\psi_{i_1}(x)^{z_1}\dots \psi_{i_q}(x)^{z_q}\right)}\\& \quad\quad\quad\quad\leq\sum_{z_i\geq 2, q<s, z_1+\dots+z_q=w-1,i_j\neq i}\frac{w!}{z_1!\dots z_q!}\abs{\lambda_{i_1}^{z_1}\dots\lambda_{i_q}^{z_q}\EE\left[\psi_{i_1}(x)^{z_1}\dots \psi_{i_q}(x)^{z_q}\right]}\\
&\quad\quad\quad\quad\lesssim \sum_{z_i\geq 2, q<s, z_1+\dots+z_q=w-1,i_j\neq i}\abs{\lambda_{i_1}^{z_1}\dots \lambda_{i_q}^{z_q}}\\
&\quad\quad\quad\quad\lesssim L^{s-1}
\end{aligned}
\end{equation}
In the second inequality, we use Gaussian hypercontractivity, Lemma \ref{lemma: gaussian hypercontractivity}.

Combining equation \eqref{equation: first term first part} and \eqref{equa: new sparse parity term number small} together, 
and noticing that \[
\abs{\sum_{i_j\neq i} \lambda_{i_1}^2\dots\lambda_{i_s}^2-\sum_{i_j} \lambda_{i_1}^2\dots\lambda_{i_s}^2} \leq s\lambda_i^2 \sum_{i_j\neq i} \lambda_{i_1}^2\dots\lambda_{i_{s-1}}^2\lesssim L^{s-1}
\] which can help us substitute $\sum_{i_j} \lambda_{i_1}^2\dots\lambda_{i_s}^2$ for $\sum_{i_j\neq i} \lambda_{i_1}^2\dots\lambda_{i_s}^2$,
we can have
\[
\begin{aligned}
&\frac{1}{L^{w/2}}\left(\sum_{i=1}^L \lambda_i \psi_{i}(x)\cP_0\left(\sum_{z_i\geq 2, q, z_1+\dots+z_q=w-1,i_j\neq i}\frac{w!}{z_1!\dots z_q!}\lambda_{i_1}^{z_1}\dots\lambda_{i_q}^{z_q}\psi_{i_1}(x)^{z_1}\dots \psi_{i_q}(x)^{z_q}\right)\right)\\
&\quad\quad\quad =\frac{1}{L^{w/2}}\left(\frac{w!}{2^s}\sum_{i_j} \lambda_{i_1}^2\dots\lambda_{i_s}^2\right)\left(\sum_{i=1}^L \lambda_i\psi_{i}(x)\left(1+K_i\right)\right)
\end{aligned}
\]
where $\sup_i\abs{K_i} \lesssim 1/L$.

\paragraph{Case II.}

Secondly we will consider the case that $w$ is even and denote $w=2s$. In that case, we observe that
\[
\begin{aligned}
&\sum_{z_i\geq 2, q, z_1+\dots+z_q=w-1,i_j\neq i}\frac{w!}{z_1!\dots z_q!}\lambda_{i_1}^{z_1}\dots\lambda_{i_q}^{z_q}\psi_{i_1}(x)^{z_1}\dots \psi_{i_q}(x)^{z_q}
\\&\quad\quad\quad=\sum_{z_i\geq 2, q<s, z_1+\dots+z_q=w-1,i_j\neq i}\frac{w!}{z_1!\dots z_q!}\lambda_{i_1}^{z_1}\dots\lambda_{i_q}^{z_q}\psi_{i_1}(x)^{z_1}\dots \psi_{i_q}(x)^{z_q}
\end{aligned}
\]
By a similar argument like equation \eqref{equa: new sparse parity term number small},
\[
\sup_{1\leq i \leq L}\abs{\sum_{z_i\geq 2, q<s, z_1+\dots+z_q=w-1,i_j\neq i}\frac{w!}{z_1!\dots z_q!}\lambda_{i_1}^{z_1}\dots\lambda_{i_q}^{z_q}\EE\left[\psi_{i_1}(x)^{z_1}\dots \psi_{i_q}(x)^{z_q}\right]}\lesssim L^{s-1}
\]
Therefore, we have the following bound for the $\cP_0$ terms in our equation \eqref{equation: expanding}.
\[
\begin{aligned}
&\frac{1}{L^{w/2}}\left(\sum_{i=1}^L \lambda_i \psi_{i}(x)\cP_0\left(\sum_{z_i\geq 2, q, z_1+\dots+z_q=w-1,i_j\neq i}\frac{w!}{z_1!\dots z_q!}\lambda_{i_1}^{z_1}\dots\lambda_{i_q}^{z_q}\psi_{i_1}(x)^{z_1}\dots \psi_{i_q}(x)^{z_q}\right)\right)\\
&\quad\quad\quad =\sum_{i=1}^L \lambda_iK_i\psi_{i}(x)
\end{aligned}
\]
where $\sup_i\abs{K_i} \lesssim 1/L$.

Then let us compute the $\cP_k$ terms. Firstly, we divide the monomials into two groups
\[
\begin{aligned}
&\sum_{z_i\geq 2,q,z_1+\dots+z_q=w,i_j}\frac{w!}{z_1!\dots z_q!}\lambda_{i_1}^{z_1}\dots\lambda_{i_q}^{z_q}\psi_{i_1}(x)^{z_1}\dots \psi_{i_q}(x)^{z_q}\\&\quad=\sum_{z_i\geq 2,2q<w,z_1+\dots+z_q=w,i_j}\frac{w!}{z_1!\dots z_q!}\lambda_{i_1}^{z_1}\dots\lambda_{i_q}^{z_q}\psi_{i_1}(x)^{z_1}\dots \psi_{i_q}(x)^{z_q}+\sum_{2q=w,i_j}\frac{w!}{2^q}\lambda_{i_1}^{2}\dots\lambda_{i_q}^{2}\psi_{i_1}(x)^{2}\dots \psi_{i_q}(x)^{2}
\end{aligned}
\]
For the first group, we have the following
\[
\begin{aligned}
&\norm{\cP_k\left(\sum_{z_i\geq 2,2q<w,z_1+\dots+z_q=w,i_j}\frac{w!}{z_1!\dots z_q!}\lambda_{i_1}^{z_1}\dots\lambda_{i_q}^{z_q}\psi_{i_1}(x)^{z_1}\dots \psi_{i_q}(x)^{z_q}\right)}{L^2(\gamma)}^2\\
&\quad\quad\quad\quad\quad \leq \norm{\sum_{z_i\geq 2,2q<w,z_1+\dots+z_q=w,i_j}\frac{w!}{z_1!\dots z_q!}\lambda_{i_1}^{z_1}\dots\lambda_{i_q}^{z_q}\psi_{i_1}(x)^{z_1}\dots \psi_{i_q}(x)^{z_q}}{L^2(\gamma)}^2\\
&\quad\quad\quad\quad\quad \leq (wL)^{\lceil w/2 \rceil-1}\sum_{z_i\geq 2,2q<w,z_1+\dots+z_q=w,i_j}\norm{\frac{w!}{z_1!\dots z_q!}\lambda_{i_1}^{z_1}\dots\lambda_{i_q}^{z_q}\psi_{i_1}(x)^{z_1}\dots \psi_{i_q}(x)^{z_q}}{L^2(\gamma)}^2\\
&\quad\quad\quad\quad\quad \lesssim L^{2\lceil w/2 \rceil-2}
\end{aligned}
\]
In the second equality we use Gaussian hypercontractivity, Lemma \ref{lemma: gaussian hypercontractivity}.

For the second group, we have that 
\[
\begin{aligned}
& \norm{\cP_k\left(\sum_{i_l}\frac{w!}{2^s}\lambda_{i_1}^{2}\dots\lambda_{i_s}^{2}\psi_{i_1}(x)^{2}\dots \psi_{{i_s}}(x)^{2}\right)}{L^2(\gamma)}^2=\norm{\sum_{i_l}\cP_k\left(\frac{w!}{2^s}\lambda_{i_1}^{2}\dots\lambda_{i_s}^{2}\psi_{i_1}(x)^{2}\dots \psi_{{i_s}}(x)^{2}\right)}{L^2(\gamma)}^2\\
&\quad\quad\quad =\left(\frac{w!}{2^s}\right)^2\sum_{i_l}\sum_{j_l}\langle\cP_k\left(\lambda_{i_1}^{2}\dots\lambda_{i_s}^{2}\psi_{i_1}(x)^{2}\dots \psi_{{i_s}}(x)^{2}\right),\cP_k\left(\lambda_{j_1}^{2}\dots\lambda_{j_s}^{2}\psi_{{j_1}}(x)^{2}\dots \psi_{{j_s}}(x)^{2}\right)\rangle_{L^2(\gamma)}\\
&\quad\quad\quad =\left(\frac{w!}{2^s}\right)^2\sum_{i_l,j_l,\{i_l\}\bigcap\{j_l\}\neq \varnothing}\langle\cP_k\left(\lambda_{i_1}^{2}\dots\lambda_{i_s}^{2}\psi_{i_1}(x)^{2}\dots \psi_{{i_s}}(x)^{2}\right),\cP_k\left(\lambda_{j_1}^{2}\dots\lambda_{j_s}^{2}\psi_{{j_1}}(x)^{2}\dots \psi_{{j_s}}(x)^{2}\right)\rangle_{L^2(\gamma)}\\
&\quad\quad\quad \leq \left(\frac{w!}{2^s}\right)^2 s^2L^{2s-1}\sup_{i_l} \norm{\cP_k\left(\lambda_{i_1}^{2}\dots\lambda_{i_s}^{2}\psi_{i_1}(x)^{2}\dots \psi_{{i_s}}(x)^{2}\right)}{L^2(\gamma)}^2\\
&\quad\quad\quad \lesssim L^{w-1}
\end{aligned}
\]
From the second line to the third line, we use the fact that if $\{i_l\}\bigcap \{j_l\} = \varnothing$, then $\cP_k\left(\lambda_{i_1}^{2}\dots\lambda_{i_s}^{2}\psi_{i_1}(x)^{2}\dots \psi_{{i_s}}(x)^{2}\right)$ and $\cP_k\left(\lambda_{j_1}^{2}\dots\lambda_{j_s}^{2}\psi_{{j_1}}(x)^{2}\dots \psi_{{j_s}}(x)^{2}\right)$ are two independent mean-zero random variables. Also, the third line to the fourth line is just counting the number of pairs of tuples with nonempty intersections.
The fourth line to the fifth line is using gaussian hypercontractivity, Lemma \ref{lemma: gaussian hypercontractivity}, to bound the moments.

In a word, we have derived for any $k\geq 2$, and any $w\geq 2$ that
\[
\norm{\frac{1}{L^{w/2}}\cP_k\left(\sum_{z_i\geq 2,q,z_1+\dots+z_q=w,i_j}\frac{w!}{z_1!\dots z_q!}\lambda_{i_1}^{z_1}\dots\lambda_{i_q}^{z_q}\psi_{i_1}(x)^{z_1}\dots \psi_{i_q}(x)^{z_q}\right)}{L^2(\gamma)}=\cO(L^{-1/2})
\]

Sum up all the derivations above, and we get the following conclusion.
\begin{lemma}
Given $k\geq 2$,
\begin{itemize}
\item When $w=2s+1$ with $s\geq 1$, we have\[
\norm{\cP_k (p(x)^w)-\frac{w!}{2^sL^s}\left(\sum_{i_j}\lambda_{i_1}^2\dots\lambda_{i_s}^2\right)p(x)}{L^2(\gamma)}=\cO(L^{-1/2})
\]

\item When $w=2s$ with $s\geq 1$, we have\[
\norm{\cP_k (p(x)^w)}{L^2(\gamma)}=\cO(L^{-1/2})
\]
\end{itemize}
\end{lemma}

Recall our $g(z)=\sum_{0\leq i\leq q }g_iz^i$.
After the projection, the feature that we get is approximately
$
\left(\sum_{s} \frac{1}{2^sL^s}(2s+1)!g_{2s+1}\left(\sum_{i_j} \lambda_{i_1}^2\dots\lambda_{i_s}^2\right)\right)p
$.
Precisely speaking, we have 
\begin{equation}\label{equation: approximate stein almost finish}
\norm{\cP_k h-\left(\sum_{s} \frac{1}{2^sL^s}(2s+1)!c_{2s+1}\left(\sum_{i_j} \lambda_{i_1}^2\dots\lambda_{i_s}^2\right)\right)p}{L^2(\gamma)} = \cO(L^{-1/2})
\end{equation}

Let's recall $\sum_i \lambda_i^2=L$, so that informally speaking, we expect $p(x)\sim \cN(0,1)$ in a limiting sense due to central limit theorem when $L$ is large and $\lambda_i$ are somehow balanced.
Again, from the main text, it is tempting to conjecture some kind of approximated Stein's Lemma like
\[
\cP_k (g\circ p) \approx \EE_{z\sim \cN(0,1)}\left[g'(z)\right] p
\]
Now we will verify this is indeed right.
In our case, the derivative of $g$ is 
$g'(z)=g_1+2g_2z+3g_3z^2+\dots+qg_qz^{q-1}$, and we can compute that
$
\EE_{z\sim \cN(0,1)}\left[g'(z)\right]
=\sum_{s}g_{2s+1}(2s+1)!!
$.
Furthermore, we have\[
L^s=\left(\sum_i\lambda_i^2\right)^s=\cO(L^{s-1})+s!\left(\sum_{i_j}\lambda_{i_1}^2\dots\lambda_{i_s}^2\right)
\]
And as a direct consequence, we have 
\[
\frac{1}{2^sL^s}(2s+1)!g_{2s+1}\left(\sum_{i_j} \lambda_{i_1}^2\dots \lambda_{i_s}^2\right)=(2s+1)!!g_{2s+1}+\cO(L^{-1})
\]
Simply plugging the above equation in equation \eqref{equation: approximate stein almost finish}, we get our final result.
\end{proof}

\begin{proof}[Proof of Lemma \ref{lemma: appendix A second projection}]
Firstly, we compute the hermite degree $m$ components of $p(x)^w$, $w\geq 2$. From the definition of $\cP_m$ and multinomial theorem, we know
\[
\cP_m(p(x)^w)=\frac{1}{L^{w/2}}\cP_m\left(\sum_{z_i\geq 2,q,z_1+\dots+z_q=w,i_j}\frac{w!}{z_1!\dots z_q!}\lambda_{i_1}^{z_1}\dots\lambda_{i_q}^{z_q}\psi_{i_1}(x)^{z_1}\dots \psi_{i_q}(x)^{z_q}\right)
\]
by expanding $\left(\frac{1}{\sqrt{L}}\left(\sum_{1\leq i \leq L} \lambda_i \psi_i(x)\right)\right)^w$ and computing the projection for each term. In the above expansion, if there is one index $i_1$ with exponent 1, then we get a $\psi_{i_1}(x)\prod_{j\geq 2} \psi_{{i_j}}(x)^{z_j}$ term. By \Cref{lem:hermite_product}, this term is a polynomial with Hermite degree at least $k$. As a result, \[
\cP_m\left( \psi_{i_1}(x)\prod_{j\geq 2} \psi_{{i_j}}(x)^{z_j}\right)=0.
\] This is because $\psi_i(x)$ only depends on $v_{i,1}^{\top}x,\dots,v_{i,J_i}^{\top}x$ and $\{v_{i,j}\}_{i\in [L],j\in [J_i]}$ are orthogonal vectors. 

Firstly, notice that
\[
\begin{aligned}
&\sum_{z_i\geq 2,q,z_1+\dots+z_q=w,i_j}\frac{w!}{z_1!\dots z_q!}\lambda_{i_1}^{z_1}\dots\lambda_{i_q}^{z_q}\psi_{i_1}(x)^{z_1}\dots \psi_{i_q}(x)^{z_q}\\&\quad=\sum_{z_i\geq 2,2q<w,z_1+\dots+z_q=w,i_j}\frac{w!}{z_1!\dots z_q!}\lambda_{i_1}^{z_1}\dots\lambda_{i_q}^{z_q}\psi_{i_1}(x)^{z_1}\dots \psi_{i_q}(x)^{z_q}+\sum_{2q=w,i_j}\frac{w!}{2^q}\lambda_{i_1}^{2}\dots\lambda_{i_q}^{2}\psi_{i_1}(x)^2\dots \psi_{i_q}(x)^2
\end{aligned}
\]
For the first term, we have the following estimation
\[
\begin{aligned}
&\norm{\cP_m\left(\sum_{z_i\geq 2,2q<w,z_1+\dots+z_q=w,i_j}\frac{w!}{z_1!\dots z_q!}\lambda_{i_1}^{z_1}\dots\lambda_{i_q}^{z_q}\psi_{i_1}(x)^{z_1}\dots \psi_{i_q}(x)^{z_q}\right)}{L^2(\gamma)}^2\\
&\quad\quad\quad\quad\quad \leq \norm{\sum_{z_i\geq 2,2q<w,z_1+\dots+z_q=w,i_j}\frac{w!}{z_1!\dots z_q!}\lambda_{i_1}^{z_1}\dots\lambda_{i_q}^{z_q}\psi_{i_1}(x)^{z_1}\dots \psi_{i_q}(x)^{z_q}}{L^2(\gamma)}^2\\
&\quad\quad\quad\quad\quad \lesssim d^{\lceil w/2 \rceil-1}\sum_{z_i\geq 2,2q<w,z_1+\dots+z_q=w,i_j}\norm{\frac{w!}{z_1!\dots z_q!}\lambda_{i_1}^{z_1}\dots\lambda_{i_q}^{z_q}\psi_{i_1}(x)^{z_1}\dots \psi_{i_q}(x)^{z_q}}{L^2(\gamma)}^2\\
&\quad\quad\quad\quad\quad \lesssim d^{2\lceil w/2 \rceil-2}
\end{aligned}
\]
From the third line to the fourth line we use Gaussian hypercontractivity, Lemma \ref{lemma: gaussian hypercontractivity} in Appendix \ref{subappendix: gaussian hypercontractivity} to bound the high order moments of hermite polynomials.
And for the second term, we only need to consider the case that $w=2s$ is even. In that case,
\[
\begin{aligned}
& \norm{\cP_m\left(\sum_{i_l}\frac{w!}{2^s}\lambda_{i_1}^{2}\dots\lambda_{i_s}^{2}\psi_{i_1}(x)^2 \dots \psi_{i_s}(x)^2\right)}{L^2(\gamma)}^2=\norm{\sum_{i_l}\cP_m\left(\frac{w!}{2^s}\lambda_{i_1}^{2}\dots\lambda_{i_s}^{2}\psi_{i_1}(x)^2 \dots \psi_{i_s}(x)^2\right)}{L^2(\gamma)}^2\\
&\quad\quad =\left(\frac{w!}{2^s}\right)^2\sum_{i_l}\sum_{j_l}\langle\cP_m\left(\lambda_{i_1}^{2}\dots\lambda_{i_s}^{2}\psi_{i_1}(x)^2 \dots \psi_{i_s}(x)^2\right),\cP_m\left(\lambda_{j_1}^{2}\dots\lambda_{j_s}^{2}\psi_{j_1}(x)^2\dots \psi_{j_s}(x)^2\right)\rangle_{L^2(\gamma)}\\
&\quad\quad =\left(\frac{w!}{2^s}\right)^2\sum_{i_l,j_l,\{i_l\}\bigcap\{j_l\}\neq \varnothing}\langle\cP_m\left(\lambda_{i_1}^{2}\dots\lambda_{i_s}^{2}\psi_{i_1}(x)^2 \dots \psi_{i_s}(x)^2\right),\cP_m\left(\lambda_{j_1}^{2}\dots\lambda_{j_s}^{2}\psi_{j_1}(x)^2\dots \psi_{j_s}(x)^2\right)\rangle_{L^2(\gamma)}\\
&\quad\quad \lesssim  \sup_{i_l} \norm{\cP_m\left(\lambda_{i_1}^{2}\dots\lambda_{i_s}^{2}\psi_{i_1}(x)^2 \dots \psi_{i_s}(x)^2\right)}{L^2(\gamma)}^2 d^{2s-1}\\
&\quad\quad \lesssim d^{w-1}
\end{aligned}
\]
From the second line to the third line, we use the fact that if $\{i_l\}\bigcap \{j_l\} = \varnothing$, then $\cP_m\left(\lambda_{i_1}^{2}\dots\lambda_{i_s}^{2}\psi_{i_1}(x)^2 \dots \psi_{i_s}(x)^2\right)$ and $\cP_m\left(\lambda_{j_1}^{2}\dots\lambda_{j_s}^{2}\psi_{j_1}(x)^2\dots \psi_{j_s}(x)^2\right)$ are two independent mean-zero random variables. From the third line to the fourth line, we are just counting the number of pairs of tuples with nonempty intersection which is $\cO(d^{2s-1})$.

In a word, we have derived that
\[
\norm{\frac{1}{L^{w/2}}\cP_m\left(\sum_{z_i\geq 2,q,z_1+\dots+z_q=w,i_j}\frac{w!}{z_1!\dots z_q!}\lambda_{i_1}^{z_1}\dots\lambda_{i_q}^{z_q}\psi_{i_1}(x)^{z_1}\dots \psi_{i_q}(x)^{z_q}\right)}{L^2(\gamma)}=\cO(L^{-1/2})
\]

Write $g(z)=\sum_{0\leq i\leq q} g_i z^i$ and sum over all the terms, and we get the desired result. 
\end{proof}

\subsection{Special Cases}
\paragraph{Orthogonal Decomposable Tensors.}

   Firstly, we will consider the case that $p(x):=\langle A, He_k(x)\rangle$ and $A$ is an orthogonal decomposable tensor
   \[
   A=\frac{1}{\sqrt{L}}\left(\sum_{i=1}^L \lambda_i v_i^{\otimes k}\right)
   \]
   where $\langle v_i,v_j \rangle = \delta_{ij}$. Using identities for the Hermite polynomials (\Cref{app: hermite polynomial}), one can rewrite the feature as
\[
p(x)=\frac{1}{\sqrt{L}}\left(\sum_{i=1}^L\lambda_i \langle v_i^{\otimes k}, He_k(x)\rangle\right)=\frac{1}{\sqrt{L}}\left(\sum_{i=1}^L \lambda_i h_k(v_i^{\top}x)\right)
\]
This kind of feature satisfies \Cref{assumption: on feature p in main text} with $J_i=1$ for all $i$, if we further assume the regularity conditions $\sup_{i}\abs{\lambda_i} = \cO(1)$ and $\sum_i \lambda_i^2 = L$.
\paragraph{Sum of Sparse Parities.}
Secondly, we will consider the case that \[
    A=\frac{1}{\sqrt{L}}\left(\sum_{i=1}^L \lambda_i\cdot v_{i,1}\otimes\dots\otimes v_{i,k}\right)
    \]
    where $\langle v_{i_1,j_1}, v_{i_2,j_2}\rangle = \delta_{i_1i_2}\delta_{j_1j_2}$.
    In that case, our feature can be rewritten as
\[
p(x)=\frac{1}{\sqrt{L}}\left(\sum_{i=1}^L \lambda_i \langle v_{i,1}\otimes\dots\otimes v_{i,k},He_k(x)\rangle\right)=\frac{1}{\sqrt{L}}\left(\sum_{i=1}^L \lambda_i\left(\prod_{j=1}^k \langle v_{i,j},x\rangle\right)\right)
\]
This kind of feature also satisfies \Cref{assumption: on feature p in main text} with $J_i=k$ for all $i$, if we further assume the regularity conditions $\sup_i\abs{\lambda_i} = \cO(1)$ and $\sum_i \lambda_i^2 = L$.

For a concrete example, when $v_{i,j}=e_{k(i-1)+j}$ and $L=d/k$, \[p(x)=\frac{1}{\sqrt{d/k}}\left(\lambda_1x_1x_2\dots x_k + \dots + \lambda_{d/k}x_{d-k+1}\dots x_{d}\right)\]
and hence the name ``sum of sparse parities".

\section{Proof of \Cref{thm:kernel_stage1}}
\label{sec:stage1_proofs}

The goal in this appendix is to prove \Cref{thm:kernel_stage1}, which is restated below:
\stageone*

\paragraph{Proof Outline.}

Throughout the first stage of \Cref{alg:layerwise}, $c$ remains at $0$. Consequently, during this stage, the network is given by
\[ g_{u,s,V}(x) = u^{\top} \sigma_1(Vx+s) \]
where $\sigma_1$ is a degree $k$ polynomial. Given that $V,s$ is kept constant and only $u$ is trained, the network is equivalent to a random feature model with the random feature $\sigma_1(Vx+s)$.

The proof proceeds in three steps:

\begin{itemize}
    \item First, we show that there exists $u^*$ such that $g_{u^*,s, V}$ approximates $\cP_kh$, the degree $k$ component of the target.
    \item Next, we leverage strong convexity of the empirical loss minimization problem to show that GD can find an approximate global minimizer in polynomial time.
    \item Finally, we invoke a kernel Rademacher complexity argument to bound the test performance.
\end{itemize}
In this section, we may use $\sigma(\cdot)$ to refer $\sigma_1(\cdot)$, and $m$ to refer $m_1$ due to notation simplicity.


\subsection{Approximation}

First, we show that when $\sigma$ is a $k$ degree polynomial, the random feature model can and only can approximate the degree $\leq k$ part of the target function.

\begin{lemma}\label{lemma: only capture k degree}
For any $u\in \RR^{m}$, we have the following equality for any function $h\in L^2(\RR^d,\gamma)$
\[
\norm{g_{u,s,V}-h}{L^2(\gamma)}^2=\norm{g_{u,s,V}-\cP_{\leq k} h}{L^2(\gamma)}^2+\norm{\cP_{\leq k} h-h}{L^2(\gamma)}^2
\]
\end{lemma}
\begin{remark} From Lemma \ref{lemma: only capture k degree}, we can see when we try to approximate $h$ using $g_{u,s,V}$, we are actually trying our best to approximate $\cP_{\leq k} h$. That is to say,
\[
\argmin_u \norm{g_{u,s,V}-h}{L_2(\gamma)}^2=\argmin_u \norm{g_{u,s,V}-\cP_{\leq k}h}{L_2(\gamma)}^2
\]
\end{remark}

\begin{proof} By a direct computation, we have
\begin{equation}
\begin{aligned}
\norm{g_{u,s,V}-h}{L^2(\gamma)}^2&=\norm{u^{\top}\sigma(Vx+s)-\sum_j \langle H_j, He_j(x)\rangle}{L^2(\gamma)}^2\\
&=\norm{u^{\top}\sigma(Vx+s)-\sum_{j\leq k}\langle H_j, He_j(x)\rangle}{L^2(\gamma)}^2+\norm{\sum_{j\geq k+1}\langle H_j, He_j(x)\rangle}{L^2(\gamma)}^2\\
&=\norm{g_{u,s,V}-\cP_{\leq k}h}{L^2(\gamma)}^2+\norm{h-\cP_{\leq k}h}{L^2(\gamma)}^2
\end{aligned}
\end{equation}
where $H_j = \EE_x\left[h(x)He_j(x)\right]$. Here we use the hermite expansion which we state in Appendix \ref{app: hermite polynomial}.
\end{proof}


We next show that $\cP_{k}h$ can be expressed by an infinite-width network by the following three lemmas.
\begin{lemma}\label{lem:infinite_width}
    There exists $f : \SS^{d-1} \rightarrow \RR$ such that
    \begin{align*}
        \EE_v[f(v)h_k(v^{\top}x)] = (\cP_kh)(x) \quad\text{and}\quad \EE_v[f(v)^2] = \cO(d^k).
    \end{align*}
    where $v$ obeys the uniform distribution on $\SS^{d-1}$.
\end{lemma}

\begin{proof}
Recall that $(\cP_k h)(x)$ can be represented as $\langle A,He_k(x)\rangle$ for some symmetric tensor $A \in (\RR^d)^{\otimes k}$.
Furthermore, observing that
\[
\EE_v \left[f(v)h_k(v^{\top}x)\right]=\langle \EE_v \left[f(v)v^{\otimes k}\right], He_k(x)\rangle
\]
by Lemma \ref{lemma: relation between hermite poly and tensor},
it suffices to solve for $u(\cdot)$ such that
$
\EE_v [f(v)v^{\otimes k}] = A$.

Let $\operatorname{Vec}:(\RR^d)^{\otimes k} \rightarrow \RR^{d^k}$ be the unfolding operator. We claim that one solution for $f$ is
\[
f(v)=\operatorname{Vec} (v^{\otimes k})^{\top}\left(\EE_{v} \operatorname{Vec}(v^{\otimes k})\operatorname{Vec}(v^{\otimes k})^{\top}\right)^{\dagger}\operatorname{Vec}(A).
\]
First, by Corollary 42 in \citet{damian2022neural}, we have
\begin{equation}\label{eq:min_singularval}
\EE_{x \sim \gamma}\left[\operatorname{Vec}(x^{\otimes k})\operatorname{Vec}(x^{\otimes k})^{\top}\right] \succeq k!\Pi_{\text{Sym}^k(\RR^d)},
\end{equation}
where $\Pi_{\text{Sym}^k(\RR^d)}$ is the projection operator onto symmetric $k$ tensors. Since $A$ is symmetric, we indeed see that
\begin{align*}
    \operatorname{Vec}\left(\EE_v \left[f(v)v^{\otimes k}\right] \right) = \EE_{v} \operatorname{Vec}(v^{\otimes k})\operatorname{Vec} (v^{\otimes k})^{\top}\left(\EE_{v} \operatorname{Vec}(v^{\otimes k})\operatorname{Vec}(v^{\otimes k})^{\top}\right)^{\dagger}\operatorname{Vec}(A) = \operatorname{Vec}(A).
\end{align*}
Plugging this back to $\EE_v \left[f(v)^2\right]$ and applying the Cauchy inequality, we get
\begin{equation}\label{equation: norm bound for u(v)}
\EE_v\left[f(v)^2\right] \leq \lambda_{\operatorname{max}}\left(\left(\EE_{v} \operatorname{Vec}(v^{\otimes k})\operatorname{Vec}(v^{\otimes k})^{\top}\right)^{\dagger}\right)\|\operatorname{Vec}(A)\|^2
\end{equation}
Therefore, to estimate the $L^2$ norm of $f(v)$ we only need to look at the spectrum of the matrix above.

For $X\sim \cN(0,I_d)$, it is clear that $YZ$ shares the same distribution with $X$, where $Y\sim \chi(d)$ and $Z\sim \operatorname{Unif}(\SS^{d-1})$ and $Y,Z$  are independent. Therefore,
\[
\EE_{X}\left[\operatorname{Vec}(X^{\otimes k})\operatorname{Vec}(X^{\otimes k})^{\top}\right]=\EE_Y\left[ Y^{2k} \right]\EE_{Z}\left[\operatorname{Vec}(Z^{\otimes k})\operatorname{Vec}(Z^{\otimes k})^{\top}\right] 
\leq d^k \EE_{Z}\left[\operatorname{Vec}(Z^{\otimes k})\operatorname{Vec}(Z^{\otimes k})^{\top}\right] \] 
due to Lemma 44 in \citet{damian2022neural}. Furthermore, we get
$
\lambda_{\operatorname{max}}\left(\left(\EE_{X}\left[\operatorname{Vec}(X^{\otimes k})\operatorname{Vec}(X^{\otimes k})^{\top}\right]\right)^{\dagger}\right) \leq \frac{1}{k!}
$ by equation \eqref{eq:min_singularval}.
Plugging this back to equation \eqref{equation: norm bound for u(v)}, we will have \[
\EE_v\left[f(v)^2\right] \leq \frac{1}{k!}d^k\|\operatorname{Vec}(A)\|^2 \lesssim d^k,
\]
where we used the fact that $\|\operatorname{Vec}(A)\|^2_2 = \|A\|^2_F = \EE\left[(\cP_k h)(x)^2\right]= \cO(1)$.
\end{proof}

\begin{lemma}\label{lemma: general activation with bias appro}
Let $s \sim \cN(0, 1)$. Then, there exists $w : \RR \rightarrow \RR$ with $\EE_s[w(s)^2] = \cO(1)$ and 
\begin{align*}
    \EE_s\left[w(s)\sigma\left(\frac{z + s}{\sqrt{2}}\right)\right] = h_k(z).
\end{align*}
\end{lemma}

\begin{proof}
    One has the following Hermite addition formula:
    \begin{align*}
        h_i\left(\frac{z + s}{\sqrt{2}}\right) = 2^{-i/2}\sum_{j = 0}^i \binom{i}{j}^{1/2}h_{i-j}(s)h_j(z).
    \end{align*}
    Thus writing $\sigma(z) = \sum_{i \geq 0}c_ih_i(z)$, we have
    \begin{align*}
        \sigma\left(\frac{z + s}{\sqrt{2}}\right) &= \sum_{i \geq 0}\sum_{j=0}^i c_i2^{-i/2}\binom{i}{j}^{1/2}h_{i-j}(s)h_j(z)\\
        &= \sum_{j \geq 0}h_j(z)\sum_{i = j}^k c_i2^{-i/2}\binom{i}{j}^{1/2}h_{i-j}(s).
    \end{align*}
    Define $w_0, \dots, w_k$ recursively by
    \begin{align*}
        w_0 &= c_k^{-1}2^{k/2}\\
        w_j &= -c_k^{-1}2^{k/2}\binom{k}{j}^{-1/2}\left( \sum_{i = 0}^{j-1}c_{k + i- j}2^{-(k + i - j)/2}\binom{k + i - j}{i}^{1/2}w_i\right).
    \end{align*}
    As a consequence, for $j \geq 1$, we have
    \begin{align*}
        0 &= \sum_{i = 0}^{j}c_{k + i- j}2^{-(k + i - j)/2}\binom{k + i - j}{i}^{1/2}w_i.
    \end{align*}
    Therefore for all $0 \le j \le k-1$, we have
    \begin{align*}
        0 &= \sum_{i = 0}^{k-j}c_{i + j}2^{-(i + j)/2}\binom{i + j}{i}^{1/2}w_i\\
        &= \sum_{i = j}^{k}c_{i}2^{-i/2}\binom{i}{j}^{1/2}w_{i-j}.
    \end{align*}
    Setting $w(s) = \sum_{i=0}^k w_ih_i(s)$, we thus have that
    \begin{align*}
        \EE_s\left[w(s)\sigma\left(\frac{z + s}{\sqrt{2}}\right)\right] &= \sum_{j \geq 0}^kh_j(z)\sum_{i = j}^k  c_i2^{-i/2}\binom{i}{j}^{1/2}w_{i-j}\\
        &= 2^{-k/2}c_k w_0 h_k(z) + \sum_{j \geq 0}^{k-1}h_j(z)\left( \sum_{i = j}^k  c_i2^{-i/2}\binom{i}{j}^{1/2}w_{i-j}\right)\\
        &= h_k(z),
    \end{align*}
    as desired. Since we regard $k$ as a constant, and we have $\sup_i\abs{c_i}=\cO(1)$ and $c_k = \Theta(1)$ due to Assumption \ref{assumption: activation function}, the norm bound follows.
\end{proof}
\begin{lemma}\label{lemma: approximation Pk final infinite width}
There exists $u:\SS^{d-1}\times \RR \rightarrow \RR$ such that
\[
\EE_{v,s}\left[u(v,s)\sigma\left(\frac{v^{\top}x+s}{\sqrt{2}}\right)\right] = (\cP_k h)(x) \text{  and  } \EE_{v,s}\left[u(v,s)^2\right] = \cO(d^k)
\]
\end{lemma}
\begin{proof}
    By Lemma \ref{lemma: general activation with bias appro}, we get $\EE_s\left[w(s)\sigma\left(\frac{z + s}{\sqrt{2}}\right)\right] = h_k(z)$ for some $\EE_s\left[w(s)^2 \right] = \cO(1)$ and $w(\cdot)$ is a $k$ degree polynomial. Substitute $z$ with $v^{\top}x$, and then use Lemma \ref{lem:infinite_width}, we have\[
    \EE_{v,s}\left[f(v)w(s)\sigma\left(\frac{v^{\top}x+s}{\sqrt{2}}\right)\right] = \EE_{v}\left[f(v)h_k(v^{\top}x)\right] = (\cP_k h)(x)
    \]
    Set $u(v,s)=f(v)w(s)$. We next bound the $L^2$ norm of $u(v,s)$ by the independence between $v$ and $s$.
    \[
    \EE\left[u(v,s)^2\right]= \EE\left[f(v)^2w(s)^2\right] = \EE\left[f(v)^2\right]\EE\left[w(s)^2\right] \lesssim d^k
\]
\end{proof}
\begin{remark}
    In the above lemma, our feature is $\sigma\left(\frac{v^Tx+s}{\sqrt{2}}\right)$ with $v$ uniformly sampled from the unit sphere and $s$ sampled from $\cN(0,1)$. This is equivalent with our feature $\sigma(v^Tx+s)$ in the main text, with $v$ uniformly sampled from the sphere of radius $\frac{1}{\sqrt{2}}$ and $s$ sampled from $\cN(0,1/2)$.
    We will use the $\sigma\left(\frac{v^Tx+s}{\sqrt{2}}\right)$ formulation in the remainder of the section without loss of generality.
\end{remark}
Next, we show that we can use this infinite width construction to construct a finite-width network that approximates $\cP_kh$.
\begin{lemma}\label{lemma: approximation for RF d dim}
For any absolute constant $\delta \in (0,1)$ and $m\in \NN^+$, with probability at least $1-\delta/8$ over the sampling of $V,s$, there exists $u^*$ such that
\[\norm{g_{u^*,s,V}-\cP_k h}{L^2(\gamma)}^2 = \cO(m^{-1}d^k) \text{    and    } \|u^*\|^2 =\cO(m^{-1}d^k)\]
\end{lemma}
\begin{remark}
    Due to Lemma \ref{lem:approximate_stein} and utilizing the lemma above, we have \[\norm{g_{u^*,s,V}-\cP_{\leq k} h}{L^2(\gamma)}^2 \lesssim d^{-1}+m^{-1}d^k\]
\end{remark}


\begin{proof}[Proof of Lemma \ref{lemma: approximation for RF d dim}]

We use Monte Carlo sampling to help us construct the $u^*$.
Let $u(\cdot,\cdot)$ be the function from \Cref{lemma: approximation Pk final infinite width}, so that  $(\cP_kh)(x)=\EE_{v,s} \left[u(v,s)\sigma\left(\frac{v^{\top}x+s}{\sqrt{2}}\right)\right]$. We sample $\Theta=\{v_i,s_i\}_{i=1}^m$ i.i.d. and set $u^*_i := \frac{1}{m}u(v_i,s_i)$. As such, one has that
\begin{equation}
\begin{aligned}
&\mathbb{E}_{\Theta} \mathbb{E}_{x}\left|g_{u^*, s,V}(x)-(\cP_kh)(x)\right|^2 
 =\mathbb{E}_{x} \mathbb{E}_{\Theta}\left|\frac{1}{m}\sum_{j=1}^m u(v_j,s_j)\sigma\left(\frac{v^{\top}_jx+s_j}{\sqrt{2}}\right)-(\cP_kh)(x)\right|^2 \\
&\quad\quad\quad\quad\quad\quad =\frac{1}{m^2} \mathbb{E}_{x} \sum_{j, l=1}^m \mathbb{E}_{\Theta}\left[\left(u(v_j,s_j)\sigma\left(\frac{v^{\top}_jx+s_j}{\sqrt{2}}\right)\right)\left(u(v_l,s_l)\sigma\left(\frac{v^{\top}_lx+s_l}{\sqrt{2}}\right)\right)\right] \\
&\quad\quad\quad\quad\quad\quad = \frac{1}{m^2} \sum_{j=1}^m \mathbb{E}_{x} \mathbb{E}_{v_j,s_j}\left[\left(u(v_j,s_j)\sigma\left(\frac{v^{\top}_jx+s_j}{\sqrt{2}}\right)\right)^2\right] \\
&\quad\quad\quad\quad\quad\quad \lesssim \frac{1}{m} \EE_{v,s} \left[f(v)^2 w(s)^2 (1+s^{2k})\right] \\&\quad\quad\quad\quad\quad\quad\lesssim \frac{1}{m}\EE_{v,s}\left[u(v,s)^2\right]
\end{aligned}
\end{equation}
and
\[
   \EE_{\Theta}\left[\frac{1}{m}\sum_{j=1}^m u(v_j,s_j)^2\right]=\EE_{v,s}\left[ u(v,s)^2 \right]
\]
Therefore, from Markov inequality, we can derive that for any constant $K>0$ we have
\begin{equation}\label{equa: appendix B 1 markov}
\PP_{\Theta}\left(\mathbb{E}\left|g_{u^*,s, V}-\cP_kh\right|^2 \geq \Theta(1)\frac{K}{m}\EE\left[ u(v,s)^2\right]\right) \leq \frac{1}{K}
\end{equation}
and
\[
\PP_{\Theta}\left(\frac{1}{m}\sum_{j=1}^m u(v_j,s_j)^2 \geq K\EE\left[ u(v,s)^2\right]\right) \leq \frac{1}{K}
\]
for some $\Theta(1)$.
Setting $1/K=\delta/16$, plugging in the bound on $\EE\left[ u(v,s)^2\right]$ from \Cref{lemma: approximation Pk final infinite width} and noting that $\|u^*\|^2 = \frac{1}{m^2}\sum_{i=1}^m u(v_j,s_j)^2$ yields the desired result.
\end{proof}

Throughout the remainder of this section, we let $\ep_1 = \Theta(1)\frac{K}{m}\EE\left[ u(v,s)^2\right]$ for notation simplicity where the $\Theta(1)$ is from equation \eqref{equa: appendix B 1 markov}. Since we see $\delta,K$ as absolute constants, we have $\ep_1=\cO(d^k/m)$.


\subsection{Empirical Performance}

Next, we focus on the concentration over the population loss given by
\[
L(u) = \norm{g_{u,s,V} - h}{L^2(\gamma)}^2
\]
evaluated at the point $ u = u^* $, which is defined in our \Cref{lemma: approximation for RF d dim}. Our primary tool for this concentration is Corollary \ref{coro: polynomial concentration}. For the sake of notational clarity, let us define 
$
\hat{L}(u) := \frac{1}{n} \sum_{i=1}^n \left(g_{u,s,V}(x_i) - h(x_i)\right)^2
$
to represent the empirical loss based on the initial dataset $\cD_1$.

\begin{lemma}
Under the setup and the results in Lemma \ref{lemma: approximation for RF d dim},
we will have with probability at least $1-\delta/4$, 
\[
\abs{\hat{L}(u^*)-L(u^*)} \lesssim \frac{1}{\sqrt{n}}
\]
\end{lemma}
\begin{proof}[Proof]
By Corollary \ref{coro: polynomial concentration}, for any $\beta>0$, we have  \[
\PP\left[\abs{\hat{L}(u^*)-L(u^*)}\geq \beta\frac{1}{\sqrt{n}}\sqrt{\operatorname{Var}\left((g_{u^*,s,V}-h)^2\right)}\right] \leq 2\exp\left(-\Theta(1) \operatorname{min}(\beta^2,\beta^{1/r})\right)
\]
Moreover,
\[
\begin{aligned}
\operatorname{Var}\left((g_{u^*,s,V}-h)^2\right) \leq \EE_x \left[(g_{u^*,s,V}(x)-h(x))^4\right] &\lesssim \left(\EE_x\left[(g_{u^*,s,V}(x)-h(x))^2\right]\right)^2\\
& \lesssim\left(\ep_1+\EE_x\left[h(x)^2\right]\right)^2\lesssim 1,
\end{aligned}
\]
where the second inequality relies on Gaussian hypercontractivity (Lemma \ref{lemma: gaussian hypercontractivity}), and the final step sets $m\geq d^{k+\alpha}$ so that $\ep_1 \lesssim 1$. Plugging this back and choosing some $\beta = \Theta(1)$ finishes the proof.
\end{proof}

Observe that during the first stage of \Cref{alg:layerwise}, we are solving the following minimization problem:
\begin{align}\label{eq:cvx_opt}
\min_{u} \hat{L}(u)+\frac{1}{2}\xi_1 \norm{u}{}^2
\end{align}
Since this problem is strongly convex and smooth, plain GD can converge to an approximate minimizer exponentially fast. The next lemma bounds the time needed to obtain a small empirical loss:

\begin{lemma}\label{lem:time-complexity}
Set $\xi_1=\frac{2m}{d^{k+\alpha}}$.
    For any $\ep_2 \in (0,1)$, let $T_1 \gtrsim m(\log m)^k\log(m/\ep_2)$. Then, when $m,n$ are larger than some absolute constant, with probability at least $1-3\delta/8$, the predictor $\hat u:= u^{(T_1)}$ satisfies
    \begin{align*}
         \hat{L}(\hat{u})\leq \ep_1+\norm{h-\cP_{\leq k} h}{L^2(\gamma)}^2+\cO(d^{-\alpha})+\cO(1)\frac{1}{\sqrt{n}} +\ep_2
    \end{align*}
    and $\|\hat u\|^2 \lesssim \frac{d^{k+\alpha}}{m}$.
\end{lemma}

\begin{proof}
If $\hat u$ is an $\ep_2$-minimizer of \eqref{eq:cvx_opt}, then we have
    \[
\hat{L}(\hat{u})+\frac{1}{2}\xi_1 \norm{\hat{u}}{}^2 \leq \hat{L}(u^*) +\frac{1}{2}\xi_1 \norm{u^*}{}^2 + \ep_2 \leq L(u^*)+\frac{1}{2}\xi_1\norm{u^*}{}^2+\cO(1)\frac{1}{\sqrt{n}}+\ep_2
\]
By choosing $\xi_1=\frac{2m}{d^{k+\alpha}}$, we get
\[
\frac{m}{d^{k+\alpha}}\norm{\hat{u}}{}^2 \lesssim \ep_1+d^{-\alpha}+\norm{h-\cP_{\leq k} h}{L^2(\gamma)}^2+\frac{1}{\sqrt{n}} +\ep_2 \lesssim 1
\]
At the same time, we will also have \[
\hat{L}(\hat{u})\leq \ep_1+\cO(d^{-\alpha})+\norm{h-\cP_{\leq k} h}{L^2(\gamma)}^2+\cO(1)\frac{1}{\sqrt{n}} +\ep_2
\]
It thus suffices to analyze the optimization problem \eqref{eq:cvx_opt}.



Clearly, this convex optimization problem is at least $2$-strongly convex. To estimate the time complexity, we also need to estimate the smoothness of our optimization objective. 
\begin{lemma}
With probability at least $1-\cO(1/m)$,
\[
\norm{\nabla \hat{L}(u_1)- \nabla \hat{L}(u_2)}{} \lesssim m(\log m)^{k}\norm{u_1-u_2}{}
\]
\end{lemma}
\begin{proof}[Proof]
We calculate the gradient out
\[
\nabla \hat{L}(u)=\frac{1}{n}\sum_{i=1}^n 2\left(u^{\top}\sigma\left(\frac{Vx_i+s}{\sqrt{2}}\right)-h(x_i)\right)\sigma\left(\frac{Vx_i+s}{\sqrt{2}}\right)
\]
and then bound the Lipschitz constant of the gradient
\[
\begin{aligned}
\norm{\nabla \hat{L}(u_1)-\nabla \hat{L}(u_2)}{} & = \norm{\frac{2}{n}\sum_{i=1}^n \langle u_1-u_2,\sigma\left(\frac{Vx_i+s}{\sqrt{2}}\right)\rangle\sigma\left(\frac{Vx_i+s}{\sqrt{2}}\right)}{}\\
& \leq \left(\frac{2}{n}\sum_{i=1}^n \norm{\sigma\left(\frac{Vx_i+s}{\sqrt{2}}\right)}{}^2\right)\norm{u_1-u_2}{}
\end{aligned}
\]
Using Corollary \ref{coro: polynomial concentration}, we have the following concentration inequality for any $\beta\geq 1$
\[
\PP\left(\abs{\frac{1}{n}\sum_{i=1}^n \sigma\left(\frac{v_j^{\top}x_i+s_j}{\sqrt{2}}\right)^2 -\EE_x\left[\sigma\left(\frac{v_j^{\top}x+s_j}{\sqrt{2}}\right)^2\right]}\geq \beta\frac{1}{\sqrt{n}}\sqrt{\operatorname{Var}\left(\sigma\left(\frac{v_j^{\top}x+s_j}{\sqrt{2}}\right)^2\right)}\right) \leq 2e^{-\Theta(1)\beta^{1/k}}
\]
Furthermore, estimating $\EE_x\left[\sigma\left(\frac{v_j^{\top}x+s_j}{\sqrt{2}}\right)^2\right]$, $\operatorname{Var}\left(\sigma\left(\frac{v_j^{\top}x+s_j}{\sqrt{2}}\right)^2\right)$ and doing union bound over all $v_j$, we get the following inequality with probability at least $1-2me^{-\Theta(1)\beta^{1/k}}$ \[
\frac{1}{n}\sum_{i=1}^n \norm{\sigma\left(\frac{Vx_i+s}{\sqrt{2}}\right)}{}^2 \lesssim \left(1+\beta\frac{1}{\sqrt{n}}\right)\sum_{j=1}^m (1+s_j^{2k})
\]
By Corollary \ref{coro: polynomial concentration} again, we can concentrate $\frac{1}{m}\sum_{j=1}^m (1+s_j^{2k})$ and get the following with probability at least $1-2e^{-\Theta(1)m^{1/2k}}$
\[\frac{1}{m}\sum_{j=1}^m (1+s_j^{2k}) \lesssim 1
\]
In that case, we choose $\beta=\Theta(1)(\log m)^k$ for some large $\Theta(1)$ and the lemma is proved.
\end{proof}
Having derived the above Lemma, using Lemma \ref{lemma: strongly convex optimization} in Appendix \ref{appe: convex opt}, we can choose the learning rate $\eta_1=\frac{1}{m(\log m)^k\Theta(1)}$
and have \[
\norm{u^{(t)}-u_{opt}}{}^2 \leq \left(1-\frac{1}{\Theta(1)m(\log m)^k}\right)^{t} \norm{u_{opt}}{}^2
\]
where $u_{opt}$ is the unique optimal solution for that optimization problem.

In addition, in order to bound the empirical performance, we also need to upper bound the gradient.
\[
\begin{aligned}
\sup_{\|u\|\leq R} \norm{\nabla \hat{L}(u)+2u}{} &\leq 2R+\frac{2}{n}\sum_{i=1}^n \norm{\sigma\left(\frac{Vx_i+s}{\sqrt{2}}\right)}{}\abs{u^{\top}\sigma\left(\frac{Vx_i+s}{\sqrt{2}}\right)-h(x_i)}\\
& \leq 2R+\frac{2}{n}\sum_{i=1}^n\left(\norm{\sigma\left(\frac{Vx_i+s}{\sqrt{2}}\right)}{}h(x_i)+\norm{\sigma\left(\frac{Vx_i+s}{\sqrt{2}}\right)}{}^2\norm{u}{}\right)\\
& \leq 2R+\frac{2}{n}\sum_{i=1}^n \left((1+R)\norm{\sigma\left(\frac{Vx_i+s}{\sqrt{2}}\right)}{}^2+h(x_i)^2\right)\\
& \lesssim (1+3R) m(\log m)^k + \frac{2}{n}\sum_{i=1}^n h(x_i)^2
\end{aligned}
\]
with probability at least $1-\cO(1/m)$. In order to bound $\frac{1}{n}\sum_i h(x_i)^2$, by Corollary \ref{coro: polynomial concentration}, we have the following for any $\beta\geq 1$
\[
\PP\left(\abs{\frac{1}{n}\sum_{i=1}^n h(x_i)^2 -\EE_x h(x)^2} \geq \beta\frac{1}{\sqrt{n}}\sqrt{\operatorname{Var}(h(x)^2)}\right) \leq 2e^{-\Theta(1)\beta^{1/r}}
\]
Therefore, by choosing $\beta=\Theta(1)(\log n)^r$ with some large $\Theta(1)$, with probability at least $1-1/n$, we have 
$
\frac{1}{n}\sum_{i=1}^n h(x_i)^2 \lesssim 1
$. In that case, we have \[
\hat{L}(u^{(t)})+\norm{u^{(t)}}{}^2 \leq \hat{L}(u_{opt})+\norm{u_{opt}}{}^2 + \sup_{\norm{u}{}\leq 2\norm{u_{opt}}{}}\norm{\nabla \hat{L}(u)+2u}{}\norm{u^{(t)}-u_{opt}}{}
\]

Since $\norm{u_{opt}}{}=\cO(1)$, $\sup_{\norm{u}{}\leq 2\norm{u_{opt}}{}}\norm{\nabla \hat{L}(u)+2u}{} = \cO(m(\log m)^k)$, if we want \[
\sup_{\norm{u}{}\leq 2\norm{u_{opt}}{}}\norm{\nabla \hat{L}(u)+2u}{}\norm{u^{(t)}-u_{opt}}{} \leq \ep_2
\]
it is sufficient to have $T_1 \gtrsim m(\log m)^k\log (m/\ep_2)$.

\end{proof}

\subsection{Uniform Generalization Bounds}
To conclude, we need to do a union bound over $u$ for our population loss $ \norm{g_{u,s,V}-h}{L^2(\gamma)}^2
$.
We first consider a truncated version of population loss, which allows us to invoke standard Rademacher complexity generalization bounds. We conclude by properly handling the truncation.
\begin{proof}[Proof of \Cref{thm:kernel_stage1}]
Let us denote $\ell_\tau(x,y)=(x-y)^2 \wedge \tau^2$. Via standard Rademacher complexity generalization bounds, detailed in 
Lemmas \ref{lemma: rademacher complexity generalization}, \ref{lemma: contraction} and \ref{lemma: rademacher for kernels}, recall that we see $\delta$ as an absolute constant, when $m,n,d$ are larger than some absolute constant, we have that with probability at least $1-\delta/16$
\[
\begin{aligned}
\sup_{\|u\| \leq M_u} \abs{\frac{1}{n}\sum_{i=1}^n \ell_{\tau}(g_{u,s,V}(x_i),h(x_i))-\EE_x\left[\ell_{\tau}(g_{u,s,V}(x),h(x))\right]}
& \lesssim 2\operatorname{Rad}_n(\cF)+\tau^2\sqrt{\frac{1}{n}}\\
& \leq 4\tau \operatorname{Rad}_n(\cG)+\tau^2\sqrt{\frac{1}{n}}\\
& \lesssim 4\tau M_u\sqrt{\frac{m}{n}}+ \tau^2\sqrt{\frac{1}{n}}
\end{aligned}
\]
where $\cG=\{g_{u,s,V}:\norm{u}{} \leq M_u\}$ and $\cF=\{\ell_{\tau}(g_{u,s,V}(\cdot),h(\cdot)):\norm{u}{}\leq M_u\}$.
The first step is just standard uniform generalization bounds for bounded function class. The second step is via contraction lemma to compute the Rademacher complexity, and the third step is a direct calculation.
So, by that bound, we can see $\EE_{x}\left[\ell_{\tau}(g_{\hat{u},s,V}(x),h(x))\right]$ is well controlled for moderate large $\tau$. Combining this with \Cref{lem:time-complexity}, with probability $1-7\delta/16$, we have\[
\EE_{x}\left[\ell_{\tau}(g_{\hat{u},s,V}(x),h(x))\right]-\norm{h-\cP_{\leq k} h}{L^2(\gamma)}^2 \lesssim \tau M_u\sqrt{\frac{m}{n}}+ \tau^2\sqrt{\frac{1}{n}}
+\ep_1+d^{-\alpha}+\frac{1}{\sqrt{n}} +\ep_2
\]
\paragraph{Dealing with the Truncation.}
Based on the above arguments,
to bound the $L^2$ generalization error, it suffices to control the quantity
\[
\EE_x\left[\left(\left(g_{\hat{u},s,V}(x)-h(x)\right)^2 \right)\mathbf{1}_{\abs{g_{\hat{u},s,V}(x)-h(x)}\geq \tau}\right]
\]

This is done in the following lemma, whose proof is deferred to \Cref{sec:prove_truncation}
\begin{lemma}\label{lem:deal_with_truncation}
With probability at least $1-\delta/32$, for any $\tau\gtrsim 1$, we have
\[
\EE_x\left[\left(\left(g_{\hat{u},s,V}(x)-h(x)\right)^2 \right)\mathbf{1}_{\abs{g_{\hat{u},s,V}(x)-h(x)}\geq \tau}\right] \lesssim e^{-\Theta(1)\tau^{2/r}}
\]
\end{lemma}

Altogether, when $m,n,d$ are larger than some absolute constant, with probability at least $1-\delta/2$, we have the following inequality
\[
\begin{aligned}
&\norm{g_{\hat{u},s,V}-h}{L^2(\gamma)}^2 -\norm{h-\cP_{\leq k} h}{L^2(\gamma)}^2\\
&\quad\quad\leq \EE_x\left[\ell_\tau(g_{\hat{u},s,V}(x),h(x))\right]-\norm{h-\cP_{\leq k} h}{L^2(\gamma)}^2+\EE_x\left[\left(\left(g_{\hat{u},s,V}(x)-h(x)\right)^2 \right)\mathbf{1}_{\abs{g_{\hat{u},s,V}(x)-h(x)}\geq \tau}\right] \\
& \quad\quad\lesssim \tau M_u\sqrt{\frac{m}{n}}+ \tau^2\sqrt{\frac{1}{n}}
+\ep_1+d^{-\alpha}+\frac{1}{\sqrt{n}} +\ep_2  + \exp\left(-\Theta(1)\tau^{2/r}\right)
\end{aligned}
\]
where we recall $\ep_1=\cO(m^{-1}d^k)$.

For any $\alpha\in (0,1)$, select $\ep_2=d^{-\alpha}$. Clearly we have $T_1=\operatorname{poly}(n,m,d)$ and $\eta_1=\frac{1}{\operatorname{poly}(n,m,d)}$ in that case. Recall that we have chosen the width $m\geq d^{k+\alpha}$, the sample size $n\geq d^{k+3\alpha}$, and we choose the truncation level to be $\tau=\Theta(1)(\log d)^{r/2}$ and $M_u^2=\Theta\left(\frac{d^{k+\alpha}}{m}\right)$. Plugging those in yields 
\begin{align*}
\norm{g_{\hat{u},s,V}-\cP_{\leq k} h}{L^2(\gamma)}^2 &\leq \norm{g_{\hat{u},s,V}-h}{L^2(\gamma)}^2-\norm{h-\cP_{\leq k} h}{L^2(\gamma)}^2 + \cO(1/d)\\
&\lesssim (\log d)^{r/2}d^{-\alpha} +(\log d)^r d^{-k/2-3\alpha/2} \\&= \widetilde{\cO}(d^{-\alpha}),
\end{align*}
as desired.
\end{proof}

\subsubsection{Proof of \Cref{lem:deal_with_truncation}}\label{sec:prove_truncation}
\begin{proof}[Proof of \Cref{lem:deal_with_truncation}]

We will first use Cauchy inequality, then estimate the moments.
\begin{equation}\label{equa: cauchy inequality truncation}
   \begin{aligned} &\left(\EE_x\left[\left(\left(g_{\hat{u},s,V}(x)-h(x)\right)^2 \right)\mathbf{1}_{\abs{g_{\hat{u},s,V}(x)-h(x)}\geq \tau}\right] \right)^2\leq \EE_x\left[(g_{\hat{u},s,V}(x)-h(x))^4\right]\PP\left(\abs{g_{\hat{u},s,V}(x)-h(x)}\geq \tau\right)\\
&\quad\quad\quad\quad\quad\quad\quad\quad  \lesssim \left(\EE_x\left[g_{\hat{u},s,V}(x)^4\right]+\EE_x\left[h(x)^4\right]\right)\PP\left(\abs{g_{\hat{u},s,V}(x)-h(x)}\geq \tau\right)\\
& \quad\quad\quad\quad\quad\quad\quad\quad  
\lesssim\left(\EE_x\left[g_{\hat{u},s,V}(x)^2\right]^2+\EE_x\left[h(x)^2\right]^2\right)\PP\left(\abs{g_{\hat{u},s,V}(x)-h(x)}\geq \tau\right)
   \end{aligned}
\end{equation}
The last step is by Gaussian hypercontractivity, Lemma \ref{lemma: gaussian hypercontractivity}.
Recall $g_{u,s,V}(x)=u^{\top}\sigma\left(\frac{Vx+s}{\sqrt{2}}\right)$. Notice that
\begin{equation}\label{equa: estimate the L2 norm for RF}
\EE_x\left[g_{u,s,V}(x)^2\right]=u^{\top}\EE_x \left[\sigma\left(\frac{Vx+s}{\sqrt{2}}\right)\sigma\left(\frac{Vx+s}{\sqrt{2}}\right)^{\top}\right]u
\end{equation}
Therefore, we just need to give a tight bound for $\hat{u}^{\top} \EE_x \left[\sigma\left(\frac{Vx+s}{\sqrt{2}}\right)\sigma\left(\frac{Vx+s}{\sqrt{2}}\right)^{\top}\right] \hat{u}$.
For notation simplicity, in this proof, we will temporarily denote $Z_i:=\sigma\left(\frac{Vx_i+s}{\sqrt{2}}\right)$, $Z:=\sigma\left(\frac{Vx+s}{\sqrt{2}}\right)$, $\Sigma := \EE_x\left[ZZ^{\top}\right]$.

Noticing that we have
\[
\frac{1}{n}\sum_{i=1}^n g_{\hat{u},s,V}(x_i)^2 \leq \frac{2}{n}\sum_{i=1}^n (g_{\hat{u},s,V}(x_i)-h(x_i))^2 + \frac{2}{n}\sum_{i=1}^n h(x_i)^2 \lesssim 1
\]
with probability at least $1-\delta/64$, due to the small training loss and some standard concentration for $\frac{1}{n}\sum_i h(x_i)^2$. That is to say, \[\hat{u}^{\top}\left(\frac{1}{n}\sum_{i=1}^n Z_iZ_i^{\top}\right)\hat{u}=\frac{1}{n}\sum_{i=1}^n \left(\hat{u}^{\top}Z_i\right)^2=\frac{1}{n}\sum_{i=1}^n g_{\hat{u},s,V}(x_i)^2 \lesssim 1
\]

Next, 
we bound the difference between $\hat{u}^{\top}\left(\frac{1}{n}\sum_{i=1}^n Z_iZ_i^{\top}\right)\hat{u}$ and $\hat{u}^{\top}\Sigma\hat{u}$. To this end,
we orthogonally decompose $\Sigma$ as $\Sigma = K^{\top}OK$, where $O$ is a diagonal matrix and $K$ is an orthogonal matrix. Write 
$
O=\operatorname{diag}\{\gamma_1,\dots,\gamma_t,0,\dots,0\}
$
for some integer $t=\operatorname{rank}(\Sigma)$, where $\gamma_i>0$ for $i\in [t]$.
Notice that $O^{1/2}=\operatorname{diag}\{\gamma_1^{1/2},\dots,\gamma_t^{1/2},0,\dots,0\}$, and we formally denote $O^{-1/2}=\operatorname{diag}\{\gamma_1^{-1/2},\dots,\gamma_t^{-1/2},0,\dots,0\}$.
Due to the fact that $\EE_x\left[KZZ^{\top}K^{\top}\right]=O$, we know $KZ$ lies in the span of $\{e_1,\dots,e_t\}$. Therefore, we have

\[
\begin{aligned}
\abs{\hat{u}^{\top}\left(\frac{1}{n}\sum_{i=1}^n Z_iZ_i^{\top}-\Sigma\right)\hat{u}} & =\abs{\hat{u}^{\top}K^{\top}O^{1/2}\left(\frac{1}{n}\sum_{i=1}^n O^{-1/2}KZ_iZ_i^{\top}K^{\top}O^{-1/2} - \begin{pmatrix}I_t&\\&0\end{pmatrix}\right)O^{1/2}K\hat{u}}\\
& \leq \hat{u}^{\top} \Sigma\hat{u}\norm{\frac{1}{n}\sum_{i=1}^n O^{-1/2}KZ_iZ_i^{\top}K^{\top}O^{-1/2} - \begin{pmatrix}I_t&\\&0\end{pmatrix}}{} 
\end{aligned}
\]
Denote $W_i:=O^{-1/2}KZ_i$ and $W:=O^{-1/2}KZ$. We see that the second moment of $W_{\leq t}$ is equal to identity matrix in $t$ dimensions: $\EE_x\left[W_{\leq t}W_{\leq t}^{\top}\right]= I_t$. That is to say, $W_{\leq t}$ is isotropic.
Next, we will bound the following operator norm
\[
\norm{\frac{1}{n}\sum_{i=1}^n O^{-1/2}KZ_iZ_i^{\top}K^{\top}O^{-1/2} - \begin{pmatrix}I_t&\\&0\end{pmatrix}}{} = \norm{\frac{1}{n}\sum_{i=1}^n W_{\leq t,i}W_{\leq t,i}^{\top} - I_t}{}
\]
by the following concentration lemma.
\begin{lemma}\label{lem:concentrate}
Let $W=W(x)\in \RR^m$ be a random vector which is a function of $x\sim\gamma$. Assume for each $i\in [m]$, the $i$-th coordinate $W_i$ is a $k$ degree polynomial w.r.t. $x$. Also assume $\EE_x\left[WW^{\top}\right]=I$. Let $W_1,\dots,W_n$ be i.i.d. generated samples. Then with probability at least $1-\delta/64$, we have\[
\operatorname{max}_{1\leq j\leq m}\abs{s_j(\widetilde{W})-\sqrt{n}} \lesssim \sqrt{m\log m (\log n)^k}
\]
where $\widetilde{W}=(W_1,\dots,W_n)^{\top}$ and $s_j$ is the singular value.
\end{lemma}
\begin{proof}
For any $z \geq \sqrt{\operatorname{Var}(\norm{W}{}^2)}$, we have the following estimation for the tail probability
\[
\begin{aligned}
\PP\left(\operatorname{max}_{1\leq i \leq n}\norm{W_i}{}^2 \geq z + m\right)
 & \leq n \PP\left(\norm{W}{}^2 \geq z + m\right)\\
 & \leq n\PP\left(\norm{W}{}^2 -\EE_x\left[\norm{W}{}^2\right] \geq z\right)\\
 & \leq 2n\operatorname{exp}\left(-\Theta(1)\left(\frac{z}{\sqrt{\operatorname{Var}(\norm{W}{}^2)}}\right)^{1/k}\right)
\end{aligned}
\]
due to polynomial concentration, Corollory \ref{coro: polynomial concentration}, where \[
\operatorname{Var}(\norm{W}{}^2) \leq \EE_x\left[\norm{W}{}^4\right] \lesssim m\sum_{i=1}^m \EE\left[W_i^4\right]\lesssim m\sum_{i=1}^m \left(\EE\left[W_i^2\right]\right)^2 \lesssim m^2
\]
Therefore, to estimate $\EE\left[\operatorname{max}_{1\leq i\leq n}\norm{W_i}{}^2\right]$, we can choose a truncation level $\Theta(1)(\log n)^k\sqrt{\operatorname{Var}(\norm{W}{}^2)}+m$ with a large $\Theta(1)$.\[
\begin{aligned}
&\EE\left[\operatorname{max}_{1\leq i\leq n}\norm{W_i}{}^2\right] \lesssim m (\log n)^k+ \EE_x\left[\operatorname{max}_{1\leq i\leq n}\norm{W_i}{}^2 1_{\operatorname{max}_{1\leq i\leq n}\norm{W_i}{}^2 \geq \Theta(1)(\log n)^k\sqrt{\operatorname{Var}(\norm{W}{}^2)}+m}\right]\\
& \quad\quad\lesssim m(\log n)^k + \int_{\Theta(1)(\log n)^k \sqrt{\operatorname{Var}(\norm{W}{}^2)}}^{+\infty} 2 \operatorname{exp}\left(-\Theta(1)\left(\frac{z}{\sqrt{\operatorname{Var}(\norm{W}{}^2)}}\right)^{1/k}+\log n\right) dz\\
& \quad\quad \lesssim m(\log n)^k + \int_{\Theta(1)\log n}^{+\infty} \operatorname{exp}(-\Theta(1)\tilde z + \log n)\tilde z^{k-1}d\tilde z\\
&\quad\quad\lesssim m(\log n)^k
\end{aligned}
\]
 We will use the above estimation and the following Lemma from Theorem 5.45, \citet{vershynin2011introduction} to estimate the singular values of $\widetilde{W}$.
\begin{lemma}Let $A$ be an $N \times n$ matrix whose rows $A_i$ are independent isotropic random vectors in $\mathbb{R}^n$. Let $m:=$ $\mathbb{E} \max _{i \leq N}\left\|A_i\right\|_2^2$. Then
$$
\mathbb{E} \max _{j \leq n}\left|s_j(A)-\sqrt{N}\right| \lesssim \sqrt{m \log \min (N, n)}
$$

\end{lemma}

Therefore, combining that lemma and Markov inequality to gain a high probability bound, with probability at least $1-\delta/64$, we have\[
\operatorname{max}_{1\leq j\leq m}\abs{s_j(\widetilde{W})-\sqrt{n}} \lesssim \sqrt{m\log m (\log n)^k}
\]
\end{proof}
Applying \Cref{lem:concentrate} to $W_{\leq t}$, we have
\[
\norm{\frac{1}{n}\sum_{i=1}^n W_{\leq t,i}W_{\leq t,i}^{\top} - I_t}{} \lesssim \sqrt{\frac{t\log t (\log n)^k}{n}} 
\]
with probability at least $1-\delta/64$.
Next, we give an upper bound over $t$, the rank of our kernel matrix $\Sigma$.
Using the Hermite addition formula, we have\[
\sigma\left(\frac{Vx+s}{\sqrt{2}}\right) = \sum_{j=0}^k h_j(Vx) \odot A_j
\]
where $A_j\in \RR^m$ is some vector that only depends on $\sigma(\cdot)$, $j$ and $s$. Plugging that in our $\Sigma$, we have the following decomposition
\[
\begin{aligned}
\EE_x\left[\sigma\left(\frac{Vx+s}{\sqrt{2}}\right)\sigma\left(\frac{Vx+s}{\sqrt{2}}\right)^T\right]&=\EE_x\left[\left(\sum_{j=0}^k h_j(Vx) \odot A_j\right)\left(\sum_{j=0}^k h_j(Vx) \odot A_j\right)^T\right]\\
&=\sum_{j=0}^k\EE_x\left[(h_j(Vx)\odot A_j)(h_j(Vx)\odot A_j)^T\right]:=\sum_{j=0}^k \Sigma_j
\end{aligned}
\]
For each $0\leq j\leq k$, we have
\[
\Sigma_j(p,q)=A_{j,p}A_{j,q}\langle v_p^{\otimes j},v_q^{\otimes j}\rangle=\langle A_{j,p}v_p^{\otimes j},A_{j,q}v_q^{\otimes j} \rangle
\]
where $A_{j,l}$ is the $l$-th element of $A_j$, and $\Sigma_j(p,q)$ is the $(p,q)$ element of our matrix $\Sigma_j$.
Therefore, define $M_j=\left(A_{j,1}v_1^{\otimes j},\dots,A_{j,m}v_{m}^{\otimes j}\right)\in \RR^{d^j\times m}$, and we have $\Sigma_j=M_j^TM_j$ and thus $\operatorname{rank}(\Sigma_j) \leq d^j$. Therefore, $\operatorname{rank}(\Sigma) \leq \sum_{j=0}^k \operatorname{rank}(\Sigma_j)\lesssim d^k$ and $t\lesssim d^k$.

Therefore,
we have
\[
\norm{\frac{1}{n}\sum_{i=1}^n W_{\leq t,i}W_{\leq t,i}^{\top} - I_t}{} \lesssim \sqrt{\frac{t\log t (\log n)^k}{n}} \lesssim \sqrt{\frac{d^k\log d(\log n)^k}{n}}
\]
and
\[
\abs{\hat{u}^{\top}\left(\frac{1}{n}\sum_{i=1}^n Z_iZ_i^{\top}-\Sigma\right)\hat{u}} \leq \hat{u}^{\top} \Sigma\hat{u}\norm{\frac{1}{n}\sum_{i=1}^n W_{\leq t,i}W_{\leq t,i}^{\top} - I_t}{}\lesssim \sqrt{\frac{d^k\log d (\log n)^k}{n}} \hat{u}^{\top}\Sigma\hat{u}.
\]
As a consequence, we have \[
\EE\left[g_{\hat{u},s,V}(x)^2\right]=\hat{u}^{\top}\Sigma \hat{u}\lesssim  \hat{u}^{\top}\left(\frac{1}{n}\sum_{i=1}^n Z_iZ_i^{\top}\right)\hat{u}\lesssim 1
\] when $d$ is larger than some absolute constant.
Recall that
$
\EE_x\left[h(x)^2\right] =\cO(1)$ and plug everything back into equation \eqref{equa: cauchy inequality truncation},
we have 
\[
\left(\EE_x\left[\left(\left(g_{\hat{u},s,V}(x)-h(x)\right)^2 \right)\mathbf{1}_{\abs{g_{\hat{u},s,V}(x)-h(x)}\geq \tau}\right] \right)^2 \lesssim  \PP\left(\abs{g_{\hat{u},s,V}(x)-h(x)}\geq \tau\right)
\]

Therefore, we only need to bound the $\PP\left(\abs{g_{\hat{u},s,V}(x)-h(x)}\geq \tau\right)$
by polynomial concentration. From Lemma \ref{lemma:polynomial concentration}, we get 
\[
\PP\left(\abs{g_{\hat{u},s,V}(x)-h(x)} \geq \beta \sqrt{\operatorname{Var}(g_{\hat{u},s,V}(x)-h(x))}\right) \leq 2\exp(-\Theta(1) \beta^{2/r})
\]
for any $\beta>1$. Furthermore, notice that
\[
\operatorname{Var}(g_{\hat{u},s,V}(x)-h(x)) \leq \EE_x\left[(g_{\hat{u},s,V}(x)-h(x))^2\right] \lesssim \EE\left[g_{\hat{u},s,V}(x)^2\right] + \EE\left[h(x)^2\right] \lesssim 1
\]
which is from the arguments above.
Thus, for every $\tau \gtrsim 1$, we have \[
\PP\left(\abs{g_{\hat{u},s,V}(x)-h(x)}\geq \tau\right) \leq 2\exp\left(-\Theta(1)\tau^{2/r}\right)
\]
and the proof is complete.
\end{proof}

\section{Proof of Theorem \ref{thm:main_thm}}
\label{sec:stage2_proofs}
At the end of the first stage, our learner is $h_{\theta^{(T_1)}}=g_{\hat{u},s,V}$. In the second stage of our training algorithm, letting $\hat p := g_{\hat{u},s,V}$, the network becomes
\[h_{\theta}(x)=\hat{p}(x)+\sum_{i=1}^{m_2} c_i\sigma_2(a_i\hat{p}(x)+b_i)\]
with $a_i,b_i$ random and fixed and $c_i$ trainable. The network thus implements 1-D kernel regression over the new input $\hat p$ in the second stage of our training algorithm.

By \Cref{lem:stage_1_combined}, with probability $1 - \delta/2$ we have
\[
\norm{g_{\hat{u},s,V}-\cP_k h}{L^2(\gamma)}^2 =\cO((\log d)^{r/2}d^{-\alpha}) \text{ and } \norm{\cP_k h-\EE_{z\sim \cN(0,1)}\left[g'(z)\right]p}{L^2(\gamma)}^2 = \widetilde\cO(d^{-\alpha}).
\]
For notational convenience, in the remainder of this section we let $\hat p$ be an arbitrary $k$ degree polynomial satisfying the following assumption:

\begin{assumption}
    We have a $k$-degree polynomial $\hat{p}$ which satisfies
    \[
    \norm{\hat{p}-\EE_{z\sim \cN(0,1)}\left[g'(z)\right]p}{L^2(\gamma)}^2 =\cO((\log d)^{r/2}d^{-\alpha})
    \]
    where $\alpha\in (0,1)$. Also, recall that we have assumed $\EE_{z\sim \cN(0,1)}\left[g'(z)\right] =\Theta(1)$ and we denote this quantity as $C_g$.
\end{assumption}
To prove \Cref{thm:main_thm}, we condition on the event that $\hat p = g_{\hat u, V}$ satisfies this assumption, which occurs with probability $1 - \delta/2$.

In the following we may use $\sigma(\cdot)$ to denote $\sigma_2(\cdot)$, and use $m$ to refer $m_2$, for notation simplicity. The proof strategy will be very similar with the proof in Appendix \ref{sec:stage1_proofs}. We begin by constructing a low-norm solution that obtains small loss. Next, we show GD converges to an approximate minimizer. We conclude by invoking Kernel Rademacher arguments to show generalization.

\subsection{Approximation}

Define $\widetilde{g}(z)=g(\frac{1}{C_g}z) $. The target can thus be represented as $\widetilde{g}(C_gp(x))$. We will proceed using the following two steps to bound the approximation error in $L^2(\gamma)$.
\begin{itemize}
\item Step I. Bound the difference between $\widetilde{g}\circ \hat{p}$ and $\widetilde{g}\circ (C_gp)$.
\item Step II. Using a 1-D two-layer neural network to approximate the 1-D link function $\widetilde{g}$.
\end{itemize}

For step I, we have the following simple Lemma.
\begin{lemma}\label{lemma: 1-D approximation step 1} Under the assumptions above,
$\norm{\widetilde{g}\circ \hat{p}-h}{L^2(\gamma)}^2 =\cO((\log d)^{r/2}d^{-\alpha})$.
\end{lemma}
\begin{proof}[Proof of Lemma \ref{lemma: 1-D approximation step 1}]
We have that
\[
\begin{aligned}
&\norm{\widetilde{g}\circ \hat{p}-\widetilde{g}\circ (C_g p)}{L^2(\gamma)}^2 \lesssim \sum_{k=1}^q \norm{(\hat{p}(x))^k-(C_g p(x))^k}{L^2(\gamma)}^2\\
&\quad\quad\quad\quad\leq \sum_{k=1}^q \EE_x\left[(\hat{p}(x)-C_gp(x))^2(\hat{p}(x)^{k-1}+\hat{p}(x)^{k-2}(C_gp(x))+\dots+(C_gp(x))^{k-1})^2\right]\\
&\quad\quad\quad\quad\leq \sum_{k=1}^q \sqrt{\EE_x\left[(\hat{p}(x)-C_gp(x))^4\right]\EE_x\left[(\hat{p}(x)^{k-1}+\hat{p}(x)^{k-2}(C_gp(x))+\dots+(C_gp(x))^{k-1})^4\right]}\\
& \quad\quad\quad\quad\lesssim \sum_{k=1}^q \EE_x\left[(\hat{p}(x)-C_gp(x))^2\right]\EE_x\left[(\hat{p}(x)^{k-1}+\hat{p}(x)^{k-2}(C_gp(x))+\dots+(C_gp(x))^{k-1})^2\right]\\
& \quad\quad\quad\quad\lesssim \norm{\hat{p}-C_gp}{L^2(\gamma)}^2 \lesssim (\log d)^{r/2}d^{-\alpha}
\end{aligned}
\]
where the fourth inequality and the fifth inequality are due to Lemma \ref{lemma: gaussian hypercontractivity}, Gaussian hypercontractivity.
We implicitly use $\norm{\hat{p}}{L^2(\gamma)}=\cO(1)$ and $\norm{C_gp}{L^2(\gamma)}=\cO(1)$ in the fifth inequality, too.
\end{proof}

Step II relies on \Cref{lem:stage_2}, which is restated below:

\stagetwo*

\begin{proof}[Proof of \Cref{lem:stage_2}]

We will firstly control the typical value of $\hat{p}$. From Lemma \ref{lemma:polynomial concentration}, we have 
$$
\mathbb{P}\left[|\hat{p}(x)| \geq \beta \sqrt{\operatorname{Var}(\hat{p}(x))}\right] \leq 2 \exp \left(-\Theta(1) \min \left(\beta^2, \beta^{2 / k}\right)\right)
$$ for any $\beta>0$.
That is to say, when $\beta \geq 1$, with probability at least $1-2e^{-\Theta(1)\beta^{2/k}}$ we have $\abs{\hat{p}(x)} \lesssim \beta$. We implicitly use $\norm{\hat{p}}{L^2(\gamma)}=\cO(1)$ in this argument to bound $\operatorname{Var}(\hat{p}(x))$.

Next, we will use Lemma \ref{lemma: approximation in 1 dim} to give a representation for $\widetilde{g}$ in the bounded domain.
There exists $v(\cdot,\cdot)$ supported on $\{-1,1\}\times [0,2C\beta]$ such that for any $x$ satisfying $\abs{\hat{p}(x)} \leqslant C\beta$,
\[
\EE_{a,b}\left[v(a,b)\sigma(a\hat{p}(x)+b)\right] = \widetilde{g}(\hat{p}(x))-\hat{p}(x)
\]
where $a\sim \operatorname{Unif}\{-1,1\}$ and $b$ has density $\mu_b(t)$. Furthermore, recall that we have assumed 
$
\mu_b(t)\gtrsim (1+\abs{t})^{-p}
$, and we have the following estimation $\sup_{a,b}\abs{v(a,b)}=\cO(\beta^{p+q})$.

Next, we will do a Monte Carlo sampling to approximate the target.
\begin{equation}
    \begin{aligned}
        &\EE_{a,b}\EE_x\left(\frac{1}{m}\sum_{i=1}^m v(a_i,b_i)\sigma(a_i\hat{p}(x)+b_i) - (\widetilde{g}(\hat{p}(x))-\hat{p}(x))\right)^2\\
        &\quad\quad \leq \EE_{a,b}\EE_x\left(\frac{1}{m}\sum_{i=1}^m v(a_i,b_i)\sigma(a_i\hat{p}(x)+b_i) - (\widetilde{g}(\hat{p}(x))-\hat{p}(x))\right)^2 \mathbf{1}_{\abs{\hat{p}(x)}\geq C \beta}\\ & \quad\quad\quad\quad+ \EE_{a,b}\EE_x\left(\frac{1}{m}\sum_{i=1}^m v(a_i,b_i)\sigma(a_i\hat{p}(x)+b_i) - (\widetilde{g}(\hat{p}(x))-\hat{p}(x))\right)^2 \mathbf{1}_{\abs{\hat{p}(x)}\leq C \beta}
    \end{aligned}
\end{equation}
For the second term, we have 
\begin{equation}
    \begin{aligned}
        &\EE_{a,b}\EE_x\left(\frac{1}{m}\sum_{i=1}^m v(a_i,b_i)\sigma(a_i\hat{p}(x)+b_i) - (\widetilde{g}(\hat{p}(x))-\hat{p}(x))\right)^2 \mathbf{1}_{\abs{\hat{p}(x)}\leq C \beta}\\
        & \quad\quad \leq \EE_{a,b}\EE_x\left(\frac{1}{m}\sum_{i=1}^m v(a_i,b_i)\sigma(a_i\hat{p}(x)+b_i) - \EE_{a,b}\left[v(a,b)\sigma(a\hat{p}(x)+b)\right]\right)^2\\
        &\quad\quad \leq \frac{1}{m} \EE_x\EE_{a,b}\left(v(a,b)\sigma(a\hat{p}(x)+b)\right)^2\\
        &\quad\quad \leq \frac{1}{m}\cO(\beta^{2p+2q})\left(\EE_x \hat{p}(x)^2 + \EE_b b^2\right) = \frac{1}{m}\cO(\beta^{2p+2q})
    \end{aligned}
\end{equation}
Here we implicitly use the fact that $\EE_b b^2=\cO(1)$ which is from our assumptions on $\mu_b(t)$. For the first term, by Cauchy inequality, 
\[
\begin{aligned}
&\EE_{a,b}\EE_x\left(\frac{1}{m}\sum_{i=1}^m v(a_i,b_i)\sigma(a_i\hat{p}(x)+b_i) - (\widetilde{g}(\hat{p}(x))-\hat{p}(x))\right)^2 \mathbf{1}_{\abs{\hat{p}(x)}\geq C \beta}\\
&\quad\quad \leq \sqrt{\EE_{a,b,x}\left(\frac{1}{m}\sum_{i=1}^m v(a_i,b_i)\sigma(a_i\hat{p}(x)+b_i) - (\widetilde{g}(\hat{p}(x))-\hat{p}(x))\right)^4 \PP\left(\abs{\hat{p}(x)}\geq C\beta\right)}\\
& \quad\quad \lesssim e^{-\Theta(1)\beta^{2/k}}\sqrt{\EE_{a,b,x}\left(\frac{1}{m}\sum_{i=1}^m v(a_i,b_i)\sigma(a_i\hat{p}(x)+b_i) - (\widetilde{g}(\hat{p}(x))-\hat{p}(x))\right)^4 }\\
& \quad\quad \lesssim e^{-\Theta(1)\beta^{2/k}} \sqrt{\EE_{a,b,x}\left(\frac{1}{m}\sum_{i=1}^m v(a_i,b_i)\sigma(a_i\hat{p}(x)+b_i)\right)^4 + \EE_x (\widetilde{g}(\hat{p}(x))-\hat{p}(x))^4 }\\
& \quad\quad \lesssim e^{-\Theta(1)\beta^{2/k}}\sqrt{\EE_{a,b,x}\left(v(a,b)\sigma(a\hat{p}(x)+b)\right)^4+\cO(1)} \\
& \quad\quad \lesssim e^{-\Theta(1)\beta^{2/k}} \beta^{2p+2q}
\end{aligned}
\]
Here we implicitly use the fact that $\EE_b b^4 = \cO(1)$ which is again from our assumptions on $\mu_b(t)$. 
We also use gaussian hypercontractivity, Lemma \ref{lemma: gaussian hypercontractivity} to show $\EE_x (\widetilde{g}(\hat{p}(x))-\hat{p}(x))^4 =\cO(1)$. Since $\hat{p}(x)$ is a $k$ degree polynomial with Gaussian input distribution, its higher order moments can be bounded by a polynomial of its second moment which is clearly $\cO(1)$.

From the above arguments, we already derive 
\[
\EE_{a,b}\EE_x\left(\frac{1}{m}\sum_{i=1}^m v(a_i,b_i)\sigma(a_i\hat{p}(x)+b_i) - (\widetilde{g}(\hat{p}(x))-\hat{p}(x))\right)^2 \lesssim\left(\frac{1}{m}+e^{-\Theta(1)\beta^{2/k}}\right)\beta^{2p+2q}
\]
Therefore, for any absolute constant $\delta\in (0,1)$, with probability at least $1-\delta/4$ over the sampling of the random features $a_i,b_i$, using Markov inequality, we have 
\[
\EE_x\left(\frac{1}{m}\sum_{i=1}^m v(a_i,b_i)\sigma(a_i\hat{p}(x)+b_i) - (\widetilde{g}(\hat{p}(x))-\hat{p}(x))\right)^2 \lesssim \left(\frac{1}{m}+e^{-\Theta(1)\beta^{2/k}}\right)\beta^{2p+2q}
\]

Combining this with our previous result, Lemma \ref{lemma: 1-D approximation step 1}, with probability at least $1-\delta/4$ over the sampling of the random features, we can find the parameters $c^*$ in the third layer with $\sup_i |c_i^*| = \cO(\beta^{p+q}/m)$, such that
\[
L(\theta^*)=\norm{\hat{p}(x)+\sum_{i=1}^m c_i^*\sigma(a_i\hat{p}(x)+b_i)-h(x)}{L^2(\gamma)}^2 \lesssim \left(\frac{1}{m}+e^{-\Theta(1)\beta^{2/k}}\right)\beta^{2p+2q} + (\log d)^{r/2}d^{-\alpha}
\]
where $\theta^*=(a^{(0)},b^{(0)},c^*,\hat{u},V^{(0)})$.
Let us further set $\beta=\Theta(1)(\log d)^k$ where $\Theta(1)$ is some large absolute constant. Set $m=d^{\alpha}$. In this case, we will have \[
L(\theta^*) \lesssim (d^{-\alpha}+e^{-\log^2 d})(\log d)^{2k(p+q)} + (\log d)^{r/2}d^{-\alpha} \lesssim (\log d)^{r/2+2k(p+q)}d^{-\alpha}
\]
\end{proof}
\subsection{Empirical Performance}

Next we will show the existence of good estimators in our empirical landscape. Firstly, we need to concentrate the landscape at the special point $c^*$ we constructed. With a little abuse of notations, denote the empirical version of the square loss as
\[
\hat{L}(\theta) = \frac{1}{n}\sum_{j=1}^n \left(\hat{p}(x_j)+\sum_{i=1}^m c_i\sigma(a_i\hat{p}(x_j)+b_i)-h(x_j)\right)^2
\]
where we recall that $x_j\in \cD_2$ is newly generated data which is independent of $\cD_1$.

\begin{lemma}\label{lemma: concenratrtion stage2}
With probability at least $1-3\delta/8-\cO(1)d^{-\alpha}$, we will have 
\[
\hat{L}(\theta^*) \leq \frac{1}{\sqrt{n}}\cO((\log d)^{2k(p+q)})+\cO((\log d)^{r/2+2k(p+q)}d^{-\alpha})
\]
\end{lemma}
\begin{proof}[Proof of Lemma \ref{lemma: concenratrtion stage2}]

In the following, we compute the variance term.
\[
\begin{aligned}
\EE_x \left(\hat{L}(\theta^*)-L(\theta^*)\right)^2 &= \frac{1}{n}\operatorname{Var}\left(\left(\sum_{i=1}^m c_i^*\sigma(a_i\hat{p}(x)+b_i)-(h(x)-\hat{p}(x))\right)^2\right)\\
&\leq \frac{1}{n}\EE_x\left(\sum_{i=1}^m c_i^*\sigma(a_i\hat{p}(x)+b_i)-(h(x)-\hat{p}(x))\right)^4\\
& \lesssim \frac{1}{n}\left(\EE_x\left(\sum_{i=1}^m c_i^*\sigma(a_i\hat{p}(x)+b_i)\right)^4 +\EE_x (h(x))^4 + \EE_x \hat{p}(x)^4\right)\\
& \leq \frac{1}{n}\left(m^3 \sum_{i=1}^m \EE_x c_i^{*4}(a_i\hat{p}(x)+b_i)^4  +\cO(1)\right)\\
& \lesssim \frac{1}{n}\left(1+\beta^{4p+4q}\frac{1}{m}\sum_{i=1}^m (b_i^4 + \EE_x \hat{p}(x)^4)\right)\\
& \lesssim \frac{1}{n}\beta^{4p+4q}\left(1+\frac{1}{m}\sum_{i=1}^m b_i^4\right)
\end{aligned}
\]
Here are some technical arguments to bound $\frac{1}{m}\sum_{i=1}^m b_i^4$. We have
\[
\EE_b\left(\frac{1}{m}\sum_{i=1}^m b_i^4 - \EE_b b^4\right)^2 \leq \frac{1}{m}\EE_b b^8
\]
and
\[
\PP_b\left(\left(\frac{1}{m}\sum_{i=1}^m b_i^4 - \EE_b b^4\right)^2 \geq 1\right) \leq \EE_b\left(\frac{1}{m}\sum_{i=1}^m b_i^4 - \EE_b b^4\right)^2 \leq \frac{1}{m}\EE_b b^8
\]
Therefore, recall that $\EE_b b^8=\cO(1)$ based on our assumption on $\mu_b(t)$, we will have with probability $1-\cO(1)d^{-\alpha}$, $\frac{1}{m}\sum_{i=1}^m b_i^4 \lesssim 1$. In that case, we have
\[
\EE_x \left(\hat{L}(\theta^*)-L(\theta^*)\right)^2 \lesssim \frac{1}{n}\beta^{4p+4q} = \frac{1}{n}(\log d)^{4k(p+q)}
\]
Therefore, by Markov inequality, we have
$
\abs{\hat{L}(\theta^*)-L(\theta^*)}\lesssim \frac{1}{\sqrt{n}}(\log d)^{2k(p+q)}
$
with probability at least $1-\delta/8$. In this case, we have 
\[
\hat{L}(\theta^*) \lesssim \frac{1}{\sqrt{n}}(\log d)^{2k(p+q)}+(\log d)^{r/2+2k(p+q)}d^{-\alpha}
\]
\end{proof}

In the second stage of our training algorithm,
we are doing the following minimization problem
\[
\min_{c} \hat{L}(\theta)+\frac{1}{2}\xi_2 \norm{c}{}^2
\]
via vanilla GD, where $\theta=(a^{(0)},b^{(0)},c,\hat{u},V^{(0)})$. 
Since this problem is strongly convex and smooth, the optimization problem can be easily solved by plain GD. 

\begin{lemma}\label{lem:time-complexity-2}
Set $\xi_2=2$.
    For any $\ep \in (0,1)$, let $T_2 \gtrsim m\log(m/\ep)$. Then, when $m,n,d$ are larger than some absolute constant, with probability at least $1-7\delta/16$, the predictor $\hat c := c^{(T_2)}$ and $\hat{\theta}=(a^{(0)},b^{(0)},\hat{c},\hat{u},V^{(0)})$ satisfies
\[
\hat{L}(\hat{\theta}) \lesssim \frac{1}{\sqrt{n}}(\log d)^{2k(p+q)}+\ep+(\log d)^{r/2+2k(p+q)}d^{-\alpha}
\]
and
\[
\norm{\hat{c}}{}^2 \lesssim \frac{1}{\sqrt{n}}(\log d)^{2k(p+q)}+\ep+(\log d)^{r/2+2k(p+q)}d^{-\alpha}
\]
\end{lemma}

\begin{proof}
For any given threshold $\ep\in (0,1)$, assuming $\hat{c}$ is an $\ep$ minimizer of the optimization problem, then we will have
\[
\hat{L}(\hat{\theta})+\frac{1}{2}\xi_2\norm{\hat{c}}{}^2 \leq \hat{L}(\theta^*)+\frac{1}{2}\xi_2 \norm{c^*}{}^2+\ep \lesssim \frac{1}{\sqrt{n}}(\log d)^{2k(p+q)}+\ep+\left(1+\xi_2\right)(\log d)^{r/2+2k(p+q)}d^{-\alpha}
\]
Plug $\xi_2=2$ in, then we will have
\[
\hat{L}(\hat{\theta}) \lesssim \frac{1}{\sqrt{n}}(\log d)^{2k(p+q)}+\ep+(\log d)^{r/2+2k(p+q)}d^{-\alpha}
\]
and
\[
\norm{\hat{c}}{}^2 \lesssim \frac{1}{\sqrt{n}}(\log d)^{2k(p+q)}+\ep+(\log d)^{r/2+2k(p+q)}d^{-\alpha}
\]
It thus suffices to analyze the optimization problem.

Clearly, this convex optimization problem is at least 2-strongly convex. To estimate the time complexity, we also need to estimate the smoothness of our optimization objective.
\begin{lemma}\label{lemma: time complexity smoothness 2}
    With probability at least $1-\cO(1)d^{-\alpha}-2e^{-\Theta(1)n^{1/2k}}$, we have
\[
\abs{\nabla \hat{L}(c_1)-\nabla \hat{L}(c_2)} \lesssim m
\]
\end{lemma}

\begin{proof}
We first calculate the gradient
    \[
    \nabla \hat{L}(\theta)=\frac{2}{n}\sum_{j=1}^n \left(\hat{p}(x_j)+c^{\top} \sigma(a\hat{p}(x_j)+b)-h(x_j)\right)\sigma(a\hat{p}(x_j)+b)
    \]
    then bound the Lipschitz constant for the gradient
    \[
    \begin{aligned}
    \abs{\nabla \hat{L}(c_1)-\nabla \hat{L}(c_2)} &= \abs{\frac{2}{n}\sum_{j=1}^n \langle c_1-c_2, \sigma(a\hat{p}(x_j)+b)\rangle\sigma(a\hat{p}(x_j)+b)}\\
     &\leq \frac{2}{n}\sum_{j=1}^n \norm{c_1-c_2}{}\norm{\sigma(a\hat{p}(x_j)+b)}{}^2\\
     & \leq \norm{c_1-c_2}{}\left(\frac{2}{n}\sum_{j=1}^n 
     \sum_{i=1}^m (a_i\hat{p}(x_j)+b_i)^2\right)\\
     & \leq \norm{c_1-c_2}{}\left(\frac{4m}{n}\sum_{j=1}^n \hat{p}(x_j)^2 + 4\sum_{i=1}^m b_i^2\right)
    \end{aligned}
    \]
    Here are some technical arguments to estimate $\sum_i b_i^2$. We have
\[
\EE_b\left(\frac{1}{m}\sum_{i=1}^m b_i^2 - \EE_b b^2\right)^2 \leq \frac{1}{m}\EE_b b^4
\]
and
\[
\PP_b\left(\left(\frac{1}{m}\sum_{i=1}^m b_i^2 - \EE_b b^2\right)^2 \geq 1\right) \leq \EE_b\left(\frac{1}{m}\sum_{i=1}^m b_i^2 - \EE_b b^2\right)^2 \leq \frac{1}{m}\EE_b b^4
\]
Therefore, recall that  $m=d^{\alpha}$, and also $\EE_b b^4=\cO(1)$ due to our assumption on $\mu_b(t)$, we will have with probability $1-\cO(1)d^{-\alpha}$, $\frac{1}{m}\sum_{i=1}^m b_i^2 \lesssim 1$. Moreover, we can use Corollary \ref{coro: polynomial concentration} to concentrate $\sum_j \hat{p}(x_j)^2$. More concretely, we will have $\frac{1}{n}\sum_j \hat{p}(x_j)^2 \lesssim 1$ with probability at least $1-2e^{-\Theta(1)n^{1/2k}}$, since $\hat{p}(x)^2$ is a degree $2k$ polynomial and $\operatorname{Var}(\hat{p}(x)^2) \lesssim 1$ via Gaussian hypercontractivity, Lemma \ref{lemma: gaussian hypercontractivity}. Therefore, with probability at least $1-\cO(1)d^{-\alpha}-2e^{-\Theta(1)n^{1/2k}}$, we have
\[
\abs{\nabla \hat{L}(c_1)-\nabla \hat{L}(c_2)} \lesssim 1
\]
\end{proof}
Having derived the above Lemma, using Lemma \ref{lemma: strongly convex optimization} in Appendix \ref{appe: convex opt}, we can choose the learning rate $\eta_1=\frac{1}{\Theta(m)}$
and have \[
\norm{c^{(t)}-c_{opt}}{}^2 \leq \left(1-\frac{1}{\Theta(m)}\right)^{t} \norm{c_{opt}}{}^2
\]
where $c_{opt}$ is the unique optimal solution for that optimization problem.
Furthermore, we have the following
\[
\begin{aligned}
\sup_{\|c\| \leq R} \norm{\nabla \hat{L}(c)+ 2c}{} &\leq \sup_{\|c\|\leq R}\norm{\nabla \hat{L}(c)}{} +2R \\
& \leq \frac{2}{n}\sum_{j=1}^n \norm{\sigma(a\hat{p}(x_j)+b)}{}\left(\abs{\hat{p}(x_j) - h(x_j)}+R\norm{\sigma(a\hat{p}(x_j)+b)}{}\right) + 2R\\
& \leq \frac{2}{n}\sum_{j=1}^n \left((R+1)\norm{\sigma(a\hat{p}(x_j)+b)}{}^2 + (\hat{p}(x_j)-h(x_j))^2\right)+2R\\
& \leq (R+1)\cO(m)+\frac{2}{n}\sum_{j=1}^n (\hat{p}(x_j)-h(x_j))^2 + 2R
\end{aligned}
\]
with probability at least $1-\cO(1)d^{-\alpha}-2e^{-\Theta(1)n^{1/2k}}$. The last inequality follows from the same argument in Lemma \ref{lemma: time complexity smoothness 2}.
Moreover, we can use Corollary \ref{coro: polynomial concentration} to concentrate $\sum_j (\hat{p}(x_j)-h(x_j))^2$. More concretely, we will have $\frac{1}{n}\sum_j (\hat{p}(x_j)-h(x_j))^2 \lesssim 1$ with probability at least $1-2e^{-\Theta(1)n^{1/2r}}$, since $(\hat{p}(x)-h(x))^2$ is a degree $2r$ polynomial and $\operatorname{Var}\left((\hat{p}(x)-h(x)\right)^2)\lesssim 1$ via Gaussian hypercontractivity, Lemma \ref{lemma: gaussian hypercontractivity}. Therefore, with probability at least $1-\cO(1)d^{-\alpha}-2e^{-\Theta(1)n^{1/2r}}$, we have\[
\sup_{\|c\|\leq R}\norm{\nabla \hat{L}(c)+2c}{} \lesssim (R+1)m
\]

Utilizing that fact,
we have \[
\hat{L}(c^{(t)})+\norm{c^{(t)}}{}^2 \leq \hat{L}(c_{opt})+\norm{c_{opt}}{}^2 + \sup_{\norm{c}{}\leq 2\norm{c_{opt}}{}}\norm{\nabla \hat{L}(c)+2c}{}\norm{c^{(t)}-c_{opt}}{}
\]
Since $\norm{c_{opt}}{}=\cO(1)$, $\sup_{\norm{c}{}\leq 2\norm{c_{opt}}{}}\norm{\nabla \hat{L}(c)+2c}{} = \cO(m)$, if we want \[
\sup_{\norm{c}{}\leq 2\norm{c_{opt}}{}}\norm{\nabla \hat{L}(c)+2c}{}\norm{c^{(t)}-c_{opt}}{} \leq \ep_2
\]
it is sufficient to have $T_2 \gtrsim m \log(m/\ep_2)$.
\end{proof}
In addition, for any truncation level $\tau>0$, we will also have 
\[
\frac{1}{n}\sum_{j=1}^n \ell_{\tau}(h_{\hat{\theta}}(x_j),h(x_j)) \leq \hat{L}(\hat{\theta})\lesssim \frac{1}{\sqrt{n}}(\log d)^{2k(p+q)}+\ep+(\log d)^{r/2+2k(p+q)}d^{-\alpha}
\]
which we will use later. Here we recall $\ell_{\tau}(x,y):=(x-y)^2\wedge \tau^2$.
\subsection{Uniform Generalization Bounds}

To conclude, we need a uniform generalization bound over $c$ for our population loss $L(\theta)=\norm{h_{\theta}-h}{L^2(\gamma)}^2$. As in Appendix \ref{sec:stage1_proofs}, we bound the truncated loss via a Rademacher complexity argument, and deal with the truncation term later.

\begin{proof}[Proof of \Cref{thm:main_thm}]

Recall that $\ell_\tau(x, y) = (x - y)^2 \land \tau^2$.
From Lemma \ref{lemma: rademacher complexity generalization} and \ref{lemma: contraction}, with probability at least $1-\delta/32$, we will have
\[
\begin{aligned}
\sup_{\|c\| \leq M_c} \abs{\frac{1}{n}\sum_{i=1}^n \ell_{\tau}(h_{\theta}(x_i),h(x_i))-\EE_x\left[\ell_{\tau}(h_{\theta}(x),h(x))\right]}
& \leq 4\tau \operatorname{Rad}_n(\cH)+\tau^2\sqrt{\frac{\cO(1)}{n}}
\end{aligned}
\]
where $\cH:=\{h_{\theta}:\|c\|\leq M_c\}$.
Then we will compute $\operatorname{Rad}_n(\cH)$.
\begin{lemma}
With probability at least $1-\cO(1)d^{-\alpha}$ over the sampling of $a,b$, we have
\[
\operatorname{Rad}_n(\cH)\lesssim  M_c\sqrt{\frac{m}{n}}
\]
\end{lemma}
\begin{proof}[Proof]
\[
\begin{aligned}
\operatorname{Rad}_n(\cH)&=\EE_x\EE_{\xi} \left[\sup_{\norm{c}{}\leq M_c}\frac{1}{n}\sum_{j=1}^n \xi_j\left(\sum_{i=1}^m c_i\sigma(a_i\hat{p}(x_j)+b_i)\right)\right]\\
&=\frac{1}{n}\EE_x\EE_{\xi}\left[\sup_{\norm{c}{}\leq M_c}\sum_{i=1}^m c_i\left(\sum_{j=1}^n \xi_j\sigma(a_i\hat{p}(x_j)+b_i)\right)\right]\\
&\leq \frac{M_c}{n}\EE_x\EE_{\xi}\sqrt{\sum_{i=1}^m \left(\sum_{j=1}^n \xi_j\sigma(a_i\hat{p}(x_j)+b_i)\right)^2}\\
& \leq \frac{M_c}{n}\sqrt{\EE_x\EE_{\xi}\sum_{i=1}^m \left(\sum_{j=1}^n \xi_j\sigma(a_i\hat{p}(x_j)+b_i)\right)^2}\\
& = \frac{M_c}{n} \sqrt{\EE_x\left[\sum_{i=1}^m\sum_{j=1}^n \left(\sigma(a_i\hat{p}(x_j)+b_i)\right)^2\right]}\\
& \lesssim \frac{M_c}{\sqrt{n}}\sqrt{m\EE_x \hat{p}(x)^2 + \sum_{i=1}^m b_i^2}
\end{aligned}
\]
Here are some technical arguments to estimate $\sum_i b_i^2$. We have
\[
\EE_b\left(\frac{1}{m}\sum_{i=1}^m b_i^2 - \EE_b b^2\right)^2 \leq \frac{1}{m}\EE_b b^4
\]
and
\[
\PP_b\left(\left(\frac{1}{m}\sum_{i=1}^m b_i^2 - \EE_b b^2\right)^2 \geq 1\right) \leq \EE_b\left(\frac{1}{m}\sum_{i=1}^m b_i^2 - \EE_b b^2\right)^2 \leq \frac{1}{m}\EE_b b^4
\]
Therefore, recall that  $m=d^{\alpha}$, and also $\EE_b b^4\lesssim 1$ due to our assumption on $\mu_b(t)$, we will have with probability $1-\cO(1)d^{-\alpha}$, $\frac{1}{m}\sum_{i=1}^m b_i^2 \lesssim 1$. 
In that case, plugging that in, we get our Lemma.
\end{proof}
As a consequence, with probability at least $1-\delta/32-\cO(1)d^{-\alpha}$,
\[
\begin{aligned}
\sup_{\|c\| \leq M_c} \abs{\frac{1}{n}\sum_{i=1}^n \ell_{\tau}(h_{\theta}(x_i),h(x_i))-\EE_x\left[\ell_{\tau}(h_{\theta}(x),h(x))\right]}
& \lesssim 4\tau M_c\sqrt{\frac{m}{n}}+\tau^2\sqrt{\frac{1}{n}}
\end{aligned}
\]

Lastly, we also need to deal with the truncation to get a $L^2$ generalization bound.
That is to say, we need to bound
\[
\sup_{\norm{c}{} \leq M_c} \EE_x \left[(h_{\theta}(x)-h(x))^2\mathbf{1}_{\abs{h_{\theta}(x)-h(x)}\geq \tau}\right]
\]
\begin{lemma}
We will have with probability at least $1-\cO(1)d^{-\alpha}$,
\[
\sup_{\norm{c}{} \leq M_c} \EE_x \left[(h_{\theta}(x)-h(x))^2\mathbf{1}_{\abs{h_{\theta}(x)-h(x)}\geq \tau}\right] \lesssim \frac{1}{\tau^2}(1+m^4M_c^4)
\]
\end{lemma}
\begin{proof}[Proof]
By Cauchy inequality, we have
\begin{equation}
   \begin{aligned} &\left(\EE_x\left[\left(\left(h_{\theta}(x)-h(x)\right)^2 \right)\mathbf{1}_{\abs{h_{\theta}(x)-h(x)}\geq \tau}\right] \right)^2\leq \EE_x\left[(h_\theta(x)-h(x))^4\right]\PP\left(\abs{h_{\theta}(x)-h(x)}\geq \tau\right)\\
&\quad\quad\quad\quad\quad\quad\quad\quad  \lesssim \left(\EE_x\left[h_{\theta}(x)^4\right]+\EE_x\left[h(x)^4\right]\right)\PP\left(\abs{h_{\theta}(x)-h(x)}\geq \tau\right)
   \end{aligned}
\end{equation}
Recall that $\EE_x h(x)^4=\cO(1)$. In addition, we have
\[
\begin{aligned}
\EE_x\left[h_{\theta}(x)^4\right]&=\EE_x\left[\left(\sum_{i=1}^m c_i\sigma(a_i\hat{p}(x)+b_i)\right)^4\right]\\
& \leq m^3 \sum_{i=1}^m \EE_x\left[c_i^4(a_i\hat{p}(x)+b_i)^4\right]\\
& \lesssim m^4M_c^4\left(\cO(1)+\frac{1}{m}\sum_{i=1}^m b_i^4\right) \lesssim m^4M_c^4
\end{aligned}
\]
if under the high probability event $\frac{1}{m}\sum_{i=1}^m b_i^4 \lesssim 1$. Furthermore, we have
\[
\PP\left(\abs{h_{\theta}(x)-h(x)}\geq \tau\right) \leq \frac{1}{\tau^4}\EE_x\left[(h_\theta(x)-h(x))^4\right] \lesssim \frac{1}{\tau^4} (1+m^4M_c^4)
\]
Plugging this back, we will have with probability at least $1-\cO(1)d^{-\alpha}$,
\[
\sup_{\norm{c}{} \leq M_c} \EE_x \left[(h_{\theta}(x)-h(x))^2\mathbf{1}_{\abs{h_{\theta}(x)-h(x)}\geq \tau}\right] \lesssim \frac{1}{\tau^2}(1+m^4M_c^4)
\]
\end{proof}

We now combine everything together. Let us choose $\ep=d^{-\alpha}$ and $n\geq d^{k+3\alpha}$ and recall $m=d^{\alpha}$. In that case, $\|\hat{c}\|^2 = \cO((\log d)^{r/2+2k(p+q)}d^{-\alpha})$. Therefore, when $d$ is larger than some constant that is only depending on $r,p,\alpha$, we are allowed to set $M_c=(\log d)^{\Theta(1)}d^{-\alpha}$ for some large $\Theta(1)$.
In that case, we have
\[
\norm{h_{\hat{\theta}}-h}{L^2(\gamma)}^2 \lesssim (\log d)^{r/2+2k(p+q)}d^{-\alpha}+4\tau(\log d)^{\Theta(1)}d^{-\alpha}\sqrt{d^{-k-2\alpha}}+\tau^2 d^{-k/2-3\alpha/2}+\tau^{-2}(\log d)^{\Theta(1)}
\]
We will pick up our truncation level $\tau=d^{\alpha/2}$. In that case, 
for any $\alpha \in (0,1)$, we will have \[
\norm{h_{\hat{\theta}}-h}{L^2(\gamma)}^2= \cO( (\log d)^{\Theta(1)}d^{-\alpha})=\widetilde{\cO}(d^{-\alpha})
\]
\end{proof}

\section{Technical Background}

\subsection{Hermite Polynomials}\label{app: hermite polynomial}

\begin{definition}[1D Hermite polynomials] The $k$-th normalized probabilist's Hermite polynomial, $h_k: \RR \rightarrow \RR$, is the degree $k$ polynomial defined as
\begin{align}
    h_k(x) = \frac{(-1)^k}{\sqrt{k!}}\frac{\frac{d^k\mu_\beta}{dx^k}(x)}{\mu_\beta(x)},
\end{align}
where $\mu_\beta(x) = \exp(-x^2/2)/\sqrt{2\pi}$ is the density of the standard Gaussian.
\end{definition}
The first such Hermite polynomials are $$
h_0(z)=1, h_1(z)=z, h_2(z)=\frac{z^2-1}{\sqrt{2}}, h_3(z)=\frac{z^3-3 z}{\sqrt{6}}, \cdots
$$
Denote $\beta=\mathcal{N}(0,1)$ to be the standard Gaussian in 1D. A key fact is that the normalized Hermite polynomials form an orthonormal basis of $L^2(\beta)$; that is $\EE_{x \sim \beta}[h_j(x)h_k(x)] = \delta_{jk}$.

Given a $f \in L^2(\beta)$, denote by $f(z)=\sum_k \hat{f}_k h_k(z)$ be the Hermite expansion of $f$ where
$$
\hat{f}_k=\mathbb{E}_{z \sim \beta}\left[f(z) h_k(z)\right]=\frac{1}{\sqrt{2 \pi}} \int_{\mathbb{R}} f(z) h_k(z) e^{-\frac{z^2}{2}} \mathrm{d} z
$$
is the Hermite coefficient of $f$. 
 The following lemma will be useful, which can be found in Proposition 11.31 of \citet{odonnell2021analysis}.
\begin{lemma}\label{lemma: hermite inner product} Given $f, g \in L^2(\beta)$, we have for any $u, v \in \mathbb{S}^{d-1}$ that
$$
\mathbb{E}_{x \sim\gamma}\left[f(u^{\top} x) g(v^{\top} x)\right]=\sum_{k=0}^{\infty} \hat{f}_k \hat{g}_k(u^{\top} v)^k
$$
\end{lemma}

The multidimensional analog of the Hermite polynomials is \emph{Hermite tensors}:
\begin{definition}[Hermite tensors] The $k$-th Hermite tensor in dimension $d$, $He_k: \RR^d \rightarrow (\RR^{d})^{\otimes k}$, is defined as
\begin{align*}
    He_k(x) = \frac{(-1)^k}{\sqrt{k!}}\frac{\nabla^k \mu_\gamma(x)}{\mu_\gamma(x)},
\end{align*}
where $\mu_\gamma(x) = \exp(-\frac12\|x\|^2)/(2\pi)^{d/2}$ is the density of the $d$-dimensional standard Gaussian.
\end{definition}
The Hermite tensors form an orthonormal basis of $L^2(\gamma)$; that is, for any $f \in L^2(\gamma)$, one can write the Hermite expansion
\begin{align*}
    f(x) = \sum_{k \geq 0}\langle C_k(f), He_k(x)\rangle \quad \text{where} \quad C_k(f) := \EE_{x \sim \gamma}[f(x)He_k(x)].
\end{align*}

We define the Hermite projection operator as
$
(\cP_k f)(x):=\langle C_k(f), He_k(x)\rangle
$. Intuitively speaking, the operator $\cP_k$ extracts out the $k$ degree part of a function when the input distribution is standard Gaussian. Furthermore, denote $\cP_{\leq k} :=\sum_{0\leq i\leq k} \cP_i$ and $\cP_{<k}:=\sum_{0\leq i< k} \cP_i$ as the projection operator onto the span of Hermite polynomials with degree no more than $k$, and degree less than $k$.
It is clear that
$
\norm{\cP_{\leq k} f}{L^2} \leq \norm{f}{L^2}$ for any $f\in L^2(\gamma)$. This can be shown by a simple Hermite expansion for $f$.

The next lemma can be shown by direct verification.
\begin{lemma}
We have
    \[
    He_k(x)= \frac{1}{\sqrt{k!}}\EE_{z\sim \gamma}\left[(x+iz)^{\otimes k}\right].
    \]
\end{lemma}

\begin{lemma}\label{lemma: relation between hermite poly and tensor}If $\|u\|=1$, we have
$$
h_k(u^{\top}x)=\left\langle H e_k(x), u^{\otimes k}\right\rangle.
$$
\end{lemma}
\begin{proof}[Proof]
$$
\begin{aligned}
\left\langle H e_k(x), u^{\otimes k}\right\rangle & =\frac{1}{\sqrt{k!}}\left\langle\mathbb{E}_{z \sim \gamma}\left[(x+i z)^{\otimes k}\right], u^{\otimes k}\right\rangle \\
& =\frac{1}{\sqrt{k!}}\mathbb{E}_{z \sim \gamma}\left[(u^{\top}x+i(u^{\top}z))^k\right] \\
& =\frac{1}{\sqrt{k!}}\mathbb{E}_{z \sim \beta}\left[(u^{\top}x+i z)^k\right] = h_k(u^{\top}x).
\end{aligned}
$$
\end{proof}

\subsection{Gaussian Hypercontractivity}\label{subappendix: gaussian hypercontractivity}
By Holder's inequality, we have $\|X\|_{L^p} \leq\|X\|_{L^q}$ for any random variable $X$ and any $p \leq q$. The reverse inequality does not hold in general, even up to a constant. However, for some measures like Gaussian, the reverse inequality will hold for some sufficiently nice functions like polynomials. The following lemma comes from Lemma 20 in \citet{mei2021learning}.
\begin{lemma}\label{lemma: gaussian hypercontractivity 1 dim} 
     For any $\ell \in \mathbb{N}$ and $f \in L^2(\beta)$ to be a degree $\ell$ polynomial on $\mathbb{R}$ where $\beta$ is the standard Gaussian distribution, for any $q \geq 2$, we have
$$
\left(\EE_{z\sim \beta}\left[f(z)^q\right]\right)^{2/q} \leq (q-1)^{\ell} \EE_{z\sim \beta}\left[f(z)^2\right]
$$
\end{lemma}

The next Lemma is also from \citet{mei2021learning} and is designed for uniform distribution on the sphere in $d$ dimension.
\begin{lemma}\label{lemma: sphere hypercontractivity}
For any $\ell \in \mathbb{N}$ and $f \in L^2(\mathbb{S}^{d-1})$ to be a degree $\ell$ polynomial, for any $q \geq 2$, we have
$$
\left(\EE_{z\sim \operatorname{Unif}(\SS^{d-1})}\left[f(z)^q\right]\right)^{2/q} \leq (q-1)^{\ell} \EE_{z\sim \operatorname{Unif}(\SS^{d-1})}\left[f(z)^2\right]
$$
\end{lemma}
For the case where the input distribution is standard Gaussian in $d$ dimension, we shall use the next Lemma from Theorem 4.3, \citet{Prato2007WickPI}. 
\begin{lemma}\label{lemma: gaussian hypercontractivity}
For any $\ell \in \mathbb{N}$ and $f \in L^2(\gamma)$ to be a degree $\ell$ polynomial, for any $q \geq 2$, we have
$$
\mathbb{E}_{z \sim \gamma}\left[f(z)^q\right] \leq \cO_{q,\ell}(1)\left(\mathbb{E}_{z \sim \gamma}\left[f(z)^2\right]\right)^{q / 2}
$$
where we use $\cO_{q,\ell}(1)$ to denote some universal constant that only depends on $q,\ell$.
\end{lemma}
\subsection{Polynomial Concentration}
In this subsection, we will introduce several Lemmas to control the deviation of random variables which polynomially depend on some Gaussian random variables.
We will use a slightly modified version of Lemma 30 from \citet{damian2022neural}.
\begin{lemma}\label{lemma:polynomial concentration}
Let $g$ be a polynomial of degree $p$ and $x\sim \cN(0,I_d)$. Then there exists an absolute positive constant $C_p$ depending only on $p$ such that for any $\delta>0$,
$$
\mathbb{P}\left[|g(x)-\mathbb{E}[g(x)]| \geq \delta \sqrt{\operatorname{Var}(g(x))}\right] \leq 2 \exp \left(-C_p \min \left(\delta^2, \delta^{2 / p}\right)\right)
$$
\end{lemma}

Consider the case that $x=(x_1,\dots,x_n)$ and $g(x)=\frac{1}{n}\sum_i g(x_i)$, $x_i\sim_{i.i.d.} \cN(0,I_d)\in \RR^d$ and $x\in \RR^{d\times n}$. Plug them into the above Lemma, and we get the following corollary.
\begin{corollary}\label{coro: polynomial concentration}
Let $g$ be a polynomial of degree $p$ and $x_i\sim \cN(0,I_d)$, $i\in [n]$. Then there exists an absolute positive constant $C_p$ depending only on $p$ such that for any $\delta>0$,
$$
\mathbb{P}\left[|\frac{1}{n}\sum_{i=1}^n g(x_i)-\mathbb{E}[g(x)]| \geq \delta \frac{1}{\sqrt{n}}\sqrt{\operatorname{Var}(g(x))}\right] \leq 2 \exp \left(-C_p \min \left(\delta^2, \delta^{2 / p}\right)\right)
$$
\end{corollary}

\subsection{Uniform Generalization Bounds}

\begin{definition}[Rademacher complexity] The empirical Rademacher complexity of a function class $\mathcal{F}$ on finite samples is defined as
\begin{equation}
\widehat{\operatorname{Rad}}_n(\mathcal{F})=\mathbb{E}_{\xi}\left[\sup _{f \in \mathcal{F}} \frac{1}{n} \sum_{i=1}^n \xi_i f(X_i)\right]
\end{equation}
where $\xi_1,\xi_2,\dots,\xi_n$ are \iid Rademacher random variables: $\PP(\xi_i=1)=\PP(\xi_i=-1)=\half$. Let $\rad(\cF)=\EE[\erad(\cF)]$ be the population Rademacher complexity.
\end{definition}

Then we recall the uniform law of large number via Rademacher complexity, which  can be found in \citet[Theorem 4.10]{wainwright_2019}.

\begin{lemma}\label{lemma: rademacher complexity generalization}
Assume that $f$ ranges in $[0,R]$ for all $f \in \mathcal{F}$. For any $n\geqslant 1$, for any $\delta \in(0,1)$, \wp at least $1-\delta$ over the choice of the i.i.d. training set $S=\left\{X_1, \ldots, X_n\right\}$, we have
\begin{equation}\label{eqn: Rad-gen-bound}
\sup _{f \in \mathcal{F}}\left|\frac{1}{n} \sum_{i=1}^n f\left(X_i\right)-\mathbb{E} f(X)\right| \leqslant 2 \rad(\mathcal{F})+ R 
\sqrt{\frac{\log (4 / \delta)}{n}}
\end{equation}
\end{lemma}
Then we recall the contraction Lemma in \citet[Exercise 6.7.7]{vershynin2018high} to compute Rademacher complexity.

 \begin{lemma}[Contraction Lemma] \label{lemma: contraction}
 Let $\varphi_i: \mathbb{R} \mapsto \mathbb{R}$ with $i=1, \ldots, n$ be $\beta$-Lispchitz continuous. Then,
$$
\frac{1}{n} \mathbb{E}_{\xi} \sup _{f \in \mathcal{F}} \sum_{i=1}^n \xi_i \varphi_i \circ f\left(x_i\right) \leq \beta \widehat{\operatorname{Rad}}_n(\mathcal{F})
$$
\end{lemma}

Next, we try to estimate the Rademacher complexity for random feature models.
Denote $g_{u,s,V}(x)=u^{\top}\sigma\left(\frac{Vx+s}{\sqrt{2}}\right)=\sum_{i=1}^m u_i\sigma\left(\frac{v_i^{\top}x+s_i}{\sqrt{2}}\right)$ with $v_i$ i.i.d. sampled from the uniform distribution on the unit sphere, and $s_i$ i.i.d. $\cN(0,1)$ generated. $\sigma(z)$ is a $k$ degree polynomial with $\cO(1)$ coefficients. Denote our kernel function class $\cG$ as
\[
\cG:=\{g_{u,s,V}: \|u\|\leq M_u\}
\]
Then we have the following lemma for the Rademacher complexity of $\cG$.
\begin{lemma}\label{lemma: rademacher for kernels}
With probability at least $1-2e^{-\Theta(1)m^{1/2k}}$, we have the following estimation for the Rademacher complexity of function class $\cG$.
\[
\operatorname{Rad}_n(\cG) \lesssim M_u\sqrt{\frac{m}{n}}
\]
\end{lemma}
\begin{proof}[Proof]
\begin{equation}
\begin{aligned}
    \operatorname{Rad}_n(\cG)&= \EE_{x,\xi} \left[\sup_{g_\theta \in \cG} \frac{1}{n}\sum_{i=1}^n \xi_i u^{\top}\sigma\left(\frac{Vx_i+s}{\sqrt{2}}\right)\right]\\
    &=\frac{1}{n}\EE_{x,\xi}\left[\sup_{g_\theta \in \cG} u^{\top}\left(\sum_{i=1}^n \xi_i\sigma\left(\frac{Vx_i+s}{\sqrt{2}}\right)\right)\right]\\
    & \leqslant \frac{M_u}{n}\EE_{x,\xi}\left[\norm{\sum_{i=1}^n \xi_i\sigma\left(\frac{Vx_i+s}{\sqrt{2}}\right)}{2}\right]\\
    & \leq \frac{M_u}{n}\sqrt{\EE_{x,\xi}\left[\norm{\sum_{i=1}^n \xi_i\sigma\left(\frac{Vx_i+s}{\sqrt{2}}\right)}{2}^2\right]}\\
    & = \frac{M_u}{n}\sqrt{\EE_x\left[\sum_{j=1}^m \operatorname{Var}_{\xi}\left(\sum_{i=1}^n \xi_i\sigma\left(\frac{v_j^{\top}x_i+s_j}{\sqrt{2}}\right)\right)\right]}\\
    & = \frac{M_u}{\sqrt{n}}\sqrt{\EE_x\left[\sum_{j=1}^m \sigma\left(\frac{v_j^{\top}x+s_j}{\sqrt{2}}\right)^2\right]} \lesssim M_u\sqrt{m}\sqrt{\frac{\frac{1}{m}\sum_{j=1}^m (1+s_j^{2k})}{n}}
    \end{aligned}
\end{equation}
By Corollary \ref{coro: polynomial concentration}, we can concentrate $\frac{1}{m}\sum_{j=1}^m (1+s_j^{2k})$ and get 
\[\frac{1}{m}\sum_{j=1}^m (1+s_j^{2k}) \lesssim 1\] with probability at least $1-2e^{-\Theta(1)m^{1/2k}}$. Plug that in and we get our final bound.
\end{proof}

\subsection{Convex Optimization}\label{appe: convex opt}

Denote $f(x)$ as a $C^1$ function defined in $\RR^d$. Assume that
\begin{itemize}
\item There exists $m>0$ such that $f(x)-\frac{m}{2}\norm{x}{}^2$ is convex.
\item $\norm{\nabla f(x)-\nabla f(y)}{} \leq L \norm{x-y}{}$.
\end{itemize}
The following result is standard and can be found in most convex optimization textbooks like \citet{boyd_vandenberghe_2004}.
\begin{lemma}\label{lemma: strongly convex optimization}
There exists a unique $x^*$ such that $f(x^*)=\inf_x f(x)$. And if we start at the point $x^0$ and do gradient descent with learning rate $\eta$, if $\eta \leq \frac{1}{m+L}$, then we will get\[
\norm{x^k-x^*}{}^2 \leq c^k \norm{x^0-x^*}{}^2
\]
where $c=1-\eta \frac{2mL}{m+L}$.
\end{lemma}

\subsection{Univariate Approximation}

In this subsection, we use $\sigma(z)$ to denote $\operatorname{ReLU}(z)$ and set $A\geq 1$.
\begin{lemma}Let $a \sim \operatorname{Unif}(\{-1,1\})$ and let $b$ have density $\mu_b(t)$. Then there exists $v(a, b)$ supported on $\{-1,1\} \times[A,2A]$ such that for any $|x| \leq A$,
$$
\mathbb{E}_{a, b}[v(a, b) \sigma(a x+b)]=1 \quad \text { and } \quad \sup _{a, b}|v(a, b)| \leq \frac{1}{\int_{A}^{2A} t\mu_b(t)dt}
$$
\end{lemma}
\begin{proof} Let $v(a, b)=c \mathbf{1}_{b \in[A,2A]}$ where $c=\frac{1}{\int_{A}^{2A} t\mu_b(t)dt}$. Then for $|x| \leq A$,
$$
\begin{aligned}
\mathbb{E}_{a, b}[v(a, b) \sigma(a x+b)] & =c \int_A^{2A} \frac{1}{2}[\sigma(x+t)+\sigma(-x+t)] \mu_b(t) d t \\
& =c \int_A^{2A} t \mu_b(t)dt \\
& =1
\end{aligned}
$$
\end{proof}
\begin{lemma} Let $a \sim \operatorname{Unif}(\{-1,1\})$ and let $b$ have density $\mu_b(t)$. Then there exists $v(a, b)$ supported on $\{-1,1\} \times[A,2A]$ such that for any $|x| \leq A$,
$$
\mathbb{E}_{a, b}[v(a, b) \sigma(a x+b)]=x \quad \text { and } \quad \sup _{a, b}|v(a, b)| \leq \frac{1}{\int_A^{2A} \mu_b(t)db}
$$
\end{lemma}
\begin{proof} Let $v(a, b)=c a \mathbf{1}_{b \in[A,2A]}$ where $c=\frac{1}{\int_A^{2A} \mu_b(t)dt}$. Then for $|x| \leq A$,
$$
\begin{aligned}
\mathbb{E}_{a, b}[v(a, b) \sigma(a x+b)] & =c \int_A^{2A} \frac{1}{2}[\sigma(x+t)-\sigma(-x+t)] \mu_b(t) d t \\
& =c x \int_A^{2A} \mu_b(t) d t \\
& =x
\end{aligned}
$$
\end{proof}
\begin{lemma}\label{lemma: approximation in 1 dim} Let $a \sim \operatorname{Unif}(\{-1,1\})$ and let $b$ have density $\mu_b(t)$. Let $f:\RR \rightarrow \mathbb{R}$ be any $C^2$ function. Then there exists $v(a, b)$ supported on $\{-1,1\} \times [0,2A]$ such that for any $|x| \leq A$,
$$
\mathbb{E}_{a, b}[v(a, b) \sigma(a x+b)]=f(x)  $$
and
$$\sup_{a,b}\abs{v(a,b)} =\cO\left(\sup_{x\in [-A,A],k=0,1,2}\abs{f^{(k)}(x)}\left(\frac{1}{\int_A^{2A} \mu_b(t)dt}+\frac{1}{\inf_{t\in [0,A]}\mu_b(t)}\right)\right)
$$

\end{lemma}

\begin{proof} First consider $v(a, b)=\frac{\mathbf{1}_{b \in[0,A]}}{\mu_b(t)} 2 f^{\prime \prime}(-a b)$. Then when $x \geq 0$ we have the following equation by integration by parts:
$$
\begin{aligned}
& \mathbb{E}_{a, b}[v(a, b) \sigma(a x+b)] \\
& =\int_0^A\left[f^{\prime \prime}(-t) \sigma(x+t)+f^{\prime \prime}(t) \sigma(-x+t)\right] d t \\
& = x(f'(0)-f'(-A))-Af'(-A)+f(0)-f(-A)+Af'(A)-f(A)+f(x)-xf'(A)     \\
& = f(x)+C_1+C_2x
\end{aligned}
$$
where $C_1=-Af'(-A)+f(0)-f(-A)+Af'(A)-f(A)$ and $C_2=f'(0)-f'(-A)-f'(A)$. In addition when $x<0$,
$$
\begin{aligned}
& \mathbb{E}_{a ,b}[v(a, b) \sigma(a x+b)] \\
& =\int_0^A\left[f^{\prime \prime}(-t) \sigma(x+t)+f^{\prime \prime}(t) \sigma(-x+t)\right] d t \\
& = x(f'(0)-f'(-A))-Af'(-A)+f(0)-f(-A)+Af'(A)-f(A)+f(x)-xf'(A) \\
& = f(x)+C_1+C_2x
\end{aligned}
$$
so this equality is true for all $x$. We can use the previous two lemmas to subtract the $C_1+C_2x$ term. That is to say, we can set
\[
v(a,b):=-C_1\frac{1}{\int_A^{2A} t\mu_b(t)dt }\mathbf{1}_{b\in [A,2A]}-C_2\frac{a}{\int_A^{2A} \mu_b(t)dt }\mathbf{1}_{b\in [A,2A]}+\frac{1}{\mu_b(t)}\mathbf{1}_{b\in [0,A]}2f''(-ab)
\]
in order to have $\mathbb{E}_{a, b}[v(a, b) \sigma(a x+b)]=f(x) $ for any $\abs{x} \leq A$. In this case, we have \[
\sup_{a,b}\abs{v(a,b)} =\cO\left(\sup_{x\in [-A,A],k=0,1,2}\abs{f^{(k)}(x)}\left(\frac{1}{\int_A^{2A} \mu_b(t)dt}+\frac{1}{\inf_{t\in [0,A]}\mu_b(t)}\right)\right)
\]
\end{proof}
\begin{remark}
When $f$ is a polynomial and $\mu_b(t)$ has a heavy tail, $\sup_{a,b}\abs{v(a,b)}$ will only depend on $A$ polynomially. More concretely,
consider the case $f(z)=\sum_{0\leq i\leq q} c_iz^i$ where $\sup_i\abs{c_i} = \cO(1)$. In this case, we have\[
\sup_{x\in [-A,A],k=0,1,2}\abs{f^{(k)}(x)} = \cO(A^q)
\]
Furthermore, since we have assumed $\mu_b(t) \gtrsim (\abs{t}+1)^{-p}$, we have \[
\left(\frac{1}{\int_A^{2A} \mu_b(t)dt}+\frac{1}{\inf_{t\in [0,A]}\mu_b(t)}\right) = \cO(A^p)
\text{   and   }  \sup_{a,b}\abs{v(a,b)}=\cO(A^{p+q})\]
\end{remark}

\end{document}